\documentclass[lettersize,journal]{IEEEtran}
\IEEEoverridecommandlockouts
\usepackage{cite}
\usepackage{amsfonts}
\usepackage{amsthm}
\usepackage{amssymb}
\usepackage{algorithmic}
\usepackage{graphicx}
\usepackage{textcomp}
\usepackage[table]{xcolor}
\usepackage{algorithm}
\usepackage{multirow}
\usepackage{amsmath}
\usepackage{subcaption}
\usepackage{placeins} %
\usepackage{bm}
\usepackage{makecell}
\usepackage{pgfplots}
\usepgfplotslibrary{groupplots}
\pgfplotsset{compat=newest}
\usepackage[inline]{enumitem}
\usepackage{balance}
\usepackage{adjustbox}
\usepackage{kantlipsum}
\usetikzlibrary{positioning} 
\usepackage{xspace}
\usepackage{arydshln}

\usepackage{hyperref}
\usepackage[capitalise]{cleveref}

\usepackage{array}
\usepackage{stfloats}
\usepackage{url}
\usepackage{verbatim}
\newcommand{\ourscheme}{ByITFL\xspace}
\newcommand{\ourschemelo}{LoByITFL\xspace}

\newcommand{\TotalUser}{\ensuremath{n}}
\newcommand{\Byzantine}{\ensuremath{b}}
\newcommand{\Colluding}{\ensuremath{t}}
\newcommand{\Dropout}{\ensuremath{e}}
\newcommand{\partition}{\ensuremath{m}}
\newcommand{\TotalLocalIter}{\ensuremath{C}}
\newcommand{\SSthreshold}{\ensuremath{k}}

\newcommand{\Dataset}{\ensuremath{D}}

\newcommand{\model}{\ensuremath{\bm{w}}}

\newcommand{\dimension}{\ensuremath{d}}
\newcommand{\usersamples}{\ensuremath{o_i}}

\newcommand{\globalIteration}{\ensuremath{\rho}}
\newcommand{\localIteration}{\ensuremath{c}}
\newcommand{\learningRate}{\ensuremath{\eta}}

\newcommand{\trustscore}{\ensuremath{\mathrm{TS}}}
\newcommand{\modelupdate}{\ensuremath{\bm{u}}}
\newcommand{\realnormalmodelupdate}{\ensuremath{\tilde{\bm{u}}}}
\newcommand{\normalmodelupdate}{\ensuremath{\bar{\bm{u}}}}

\newcommand{\discripoly}{\ensuremath{h}}

\newcommand{\colset}{\ensuremath{\mathcal{T}}}
\newcommand{\byzset}{\ensuremath{\mathcal{B}}}
\newcommand{\honestset}{\ensuremath{\mathcal{H}}}
\newcommand{\define}{\ensuremath{\triangleq}}

\newcommand{\sumone}{\ensuremath{\Sigma_1}}
\newcommand{\sumtwo}{\ensuremath{\mathbf{\Sigma}_2}}

\newcommand{\random}{\ensuremath{\lambda}}
\DeclareMathOperator*{\argmin}{arg\,min}
\newcommand{\aggresult}{\bm{\nu}}

\renewcommand{\baselinestretch}{.971}

\newcommand{\approdegree}{\ensuremath{\tau}}

\newcommand{\randomvector}{\ensuremath{\mathbf{r}}}

\DeclareMathOperator{\AGG}{AGG}

\newtheorem{thm}{Theorem}
\newtheorem{lem}{Lemma}

\newtheorem{remark}{Remark}

\def\BibTeX{{\rm B\kern-.05em{\sc i\kern-.025em b}\kern-.08em
    T\kern-.1667em\lower.7ex\hbox{E}\kern-.125emX}}

\begin{document}

\title{Perfect Privacy for Discriminator-Based Byzantine-Resilient Federated Learning \\
\thanks{This project has received funding from the German Research Foundation (DFG) under Grant Agreement Nos. BI 2492/1-1 and WA 3907/7-1. %
Part of this work was presented at IEEE ITW'24 \cite{xia2024byzantine} and FL-AsiaCCS'25\cite{xia2024lobyitfl}.} %
}

\author{\IEEEauthorblockN{Yue Xia, Christoph Hofmeister, Maximilian Egger, Rawad Bitar} \\
\IEEEauthorblockA{
School of Computation, Information and Technology, Technical University of Munich, Munich, Germany \\
\{yue1.xia, christoph.hofmeister, maximilian.egger, rawad.bitar\}@tum.de}
}
\maketitle

\begin{abstract}
Federated learning (FL) shows great promise in large-scale machine learning but introduces new privacy and security challenges. We propose \ourscheme and \ourschemelo, two novel FL schemes that enhance resilience against Byzantine users while keeping the users' data private from eavesdroppers. To ensure privacy and Byzantine resilience, our schemes build on having a small representative dataset available to the federator and crafting a discriminator function allowing the mitigation of corrupt users' contributions. ByITFL employs Lagrange coded computing and re-randomization, making it the first Byzantine-resilient FL scheme with perfect Information-Theoretic (IT) privacy, though at the cost of a significant communication overhead. \ourschemelo, on the other hand, achieves Byzantine resilience and IT privacy at a significantly reduced communication cost, but requires a Trusted Third Party, used only in a one-time initialization phase before training.
We provide theoretical guarantees on privacy and Byzantine resilience, along with convergence guarantees and experimental results validating our findings.

\end{abstract}


\section{Introduction}
Federated learning (FL)~\cite{mcmahan2017communication} emerged as a promising paradigm enabling a central server (federator) to train neural networks on distributed private data stored at a large number of users. The training follows an iterative structure. Per iteration, the federator sends the current global model to the users, who compute local model updates based on their local data and return these updates. The federator aggregates the users' local model updates using a certain aggregation rule and uses this aggregate to update the global model. The process is repeated until the model achieves the desired performance.

Compared to centralized machine learning, FL tackles data privacy by allowing the users to keep their sensitive data locally and send only model updates to the federator. However, these local model updates, i.e., gradients or weights, still carry sensitive information about the users' data, which can be exploited, e.g., using model inversion attacks~\cite{zhu2019deep, geiping2020inverting}, to learn more information about the users' private data. To address this risk, private\footnote{Private aggregation is referred to as ``secure aggregation'' in the literature. We use the nomenclature private aggregation to avoid confusion between privacy and security (resilience) against malicious clients.} aggregation protocols~\cite{bonawitz2017practical, aono2017privacy,kairouz2021advances} are introduced, ensuring that the federator only receives the aggregated local model update necessary for the progress of the learning algorithm; hence, guaranteeing privacy~\cite{kairouz2021advances}. 

Beyond privacy, FL faces security threats from Byzantine users, named after the Byzantine generals problem~\cite{lamport1982byzantine}, who maliciously manipulate the global model through corrupt local updates. The authors of~\cite{blanchard2017machine} show the vulnerability of any simple linear aggregation rule, such as FedAvg \cite{mcmahan2017communication}, against Byzantine attacks. Even a single Byzantine user can take full control of the learning process, either causing convergence to a corrupt model or preventing convergence altogether~\cite{blanchard2017machine}. To mitigate the effect of Byzantine users, numerous \emph{robust aggregation rules} have been proposed, e.g.,~\cite{blanchard2017machine,yin2018byzantine,guerraoui2018hidden,cao2020fltrust,zhao2022fedinv,guerraoui2024robust}. On a high level, robust aggregation rules compute some statistics of the local updates to identify the outliers. The outliers are then either excluded or given low weight on the final aggregation. 

Therefore, in FL settings, there is an \emph{inherent tension between privacy and Byzantine resilience.} On the one hand, resilience requires access to individual local updates to learn their statistics. On the other hand, privacy guarantees require concealing individual updates. This inherent tension makes developing methods that are jointly secure and private a challenging task, that we will address in this work.

Many existing schemes tackle privacy and security jointly. However, none of these schemes achieve the desired Information-Theoretic (IT) privacy guarantee with Byzantine resilience in the described FL setting. For instance, some schemes rely on clustering the users into smaller groups, and therefore compromise privacy~\cite{velicheti2021secure,xhemrishi2025fedgt} by leaking aggregated gradients of clusters of users; some require multiple federators during the execution of the learning algorithm~\cite{he2020secure,hao2021efficient,ma2022shieldfl,xu2022privacy,zhong2024wvfl,dong2021oblivious,miao2024rfed}; and many give different privacy guarantees, such as computational privacy~\cite{so2020byzantine,jahani2023byzantine,roy2022eiffel,gehlhar2023SAFEFL,fereidooni2021safelearn,dong2023privacy,ben2024scionfl,hu2023efficient,zhang2024nspfl,lu2023robust,lycklama2023rofl,alebouyeh4793556privacy,dong2021oblivious,miao2024rfed,hou2024priroagg,xing2024no,yazdinejad2024robust} and differential privacy~\cite{naseri2020local,ma2022differentially,allouah2023privacy, gu2023dp}. Among these works, the closest to this work are BREA~\cite{so2020byzantine} and ByzSecAgg~\cite{jahani2023byzantine}. Both schemes use the distance-based robust aggregation rule called Krum~\cite{blanchard2017machine} and rely on secret sharing to make Krum privacy-preserving. The computational privacy guarantee in~\cite{so2020byzantine,jahani2023byzantine} comes from the use of cryptographic commitments. Those commitments ensure that Byzantine users encode the same input local model into multiple shares given to all other users. Computationally breaking the commitments reveals the value of the input local model. In addition, due to the use of Krum, both schemes leak the pairwise distances between local updates to the federator. Furthermore, BREA leaks some extra information to the federator, through not re-randomizing secret sharings, cf.~\cite{xia2024byzantine}. This is remedied in ByzSecAgg by incorporating such a re-randomization step.

While computational privacy and differential privacy have been the primary focus in most studies due to their efficiency, both come with their challenges. Computational privacy relies on computational hardness assumptions, which fail against computationally unbounded adversaries. Differential privacy (DP) obfuscates the users' local updates using noise inserted by the users, which negatively affects\footnote{IT private aggregation can be combined with DP to allow for higher privacy guarantees for a fixed model utility. The reason is that less noise would be needed since the federator only observes the aggregation of the local updates, cf.~\cite{chen2022fundamental}.} the utility of the model. Furthermore, IT privacy offers the strongest guarantee when only a limited number of entities collude to compromise privacy, withstanding even computationally unbounded adversaries and without compromising model accuracy. 

We introduce \ourscheme and \ourschemelo, two Byzantine-resilient and IT private FL schemes. We utilize ideas from FLTrust \cite{cao2020fltrust} for Byzantine resilience, where the federator possesses a small representative root dataset to obtain a trustworthy model update as a reference. This enables the computation of a trust score (TS) for each user through a \emph{discriminator function}. To enable IT privacy, we embed each user's update in a finite field and use secret sharing. To this end, we restrict the scope to polynomial discriminator functions, which are compatible with secret sharing.

\begin{figure}[!t]
    \vspace{0.2cm}
    \centering
    \hspace*{-4.9cm} %
    \begin{adjustbox}{center}
        \resizebox{0.75\textwidth}{!}{\input{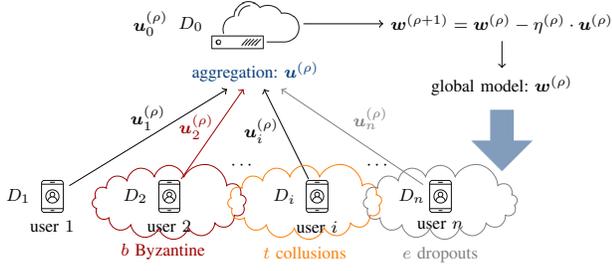}}
    \end{adjustbox}
    \caption{Federated learning system.}
    \vspace{-0.3cm}
    \label{fig:fl}
\end{figure}

\section{System Model and Preliminaries}

For an integer $\TotalUser$, we denote by $[\TotalUser]$ the set of positive integers $\{1,\cdots,n\}$. For a real number $x$, we denote by $\lfloor x \rfloor$ the largest integer less than or equal to $x$. Vectors are denoted by bold lowercase letters, e.g., $\bm{x}$. The dot product of two vectors $\bm{x}$ and $\bm{y}$ is denoted by $\langle \bm{x},\bm{y}\rangle$. Given $x_1, \cdots, x_n$ and a set $\mathcal{S} \subseteq [\TotalUser]$ of indices, we define $x_{\mathcal{S}} \define \{x_i: i \in S\}$. For two random variables $X$ and $Y$, $H(X)$ and $I(X;Y)$ denote the entropy of $X$ and the mutual information of $X$ and $Y$, respectively. The base of the logarithm will be clear from the context.

\subsection{Federated Learning Model}
We consider the FL setting with an honest-but-curious federator orchestrating the learning algorithm on the data of $\TotalUser$ users, as illustrated in \cref{fig:fl}. Among the users, up to $\Byzantine$ can be Byzantine, and up to $\Colluding$ can collude with each other (and with the federator) to infer information about other users' data. In addition, in every iteration, up to $\Dropout$ users can drop out and be unresponsive. Each user $i \in [\TotalUser]$ holds a local dataset $\Dataset_i \in \mathbb{R}^{\usersamples \times \dimension}$ drawn according to a (possibly distinct) distribution $\mathcal{D}_i$. %
As in~\cite{cao2020fltrust}, we assume that the federator possesses a small root dataset $\Dataset_0$ representing\footnote{
In practice, this small root data set can be obtained, e.g., through manually labelling a public dataset, see \cite{cao2020fltrust} for details.} the union of the users' data $\Dataset=\bigcup_{i=1}^{\TotalUser}\Dataset_i$. The goal is to find a model $\model$ that best represents the users' private data, i.e., given a loss function $F(\model, \Dataset_i)$ parametrized by $\model$, the federator wants to solve the optimization problem
\begin{equation}\label{eq:optimization}
    \min_{\model}\sum_{i=1}^{\TotalUser}F(\model, \Dataset_i).
\end{equation}

The federator employs the iterative stochastic gradient descent algorithm that starts with an initial global model $\model^{(0)}$ and local models $\model^{(0)}_1,\cdots,\model^{(0)}_n$ held by the users. At each global iteration $\globalIteration$, the federator broadcasts the current global model $\model^{(\globalIteration)}$ to all users. Each user $i\in [n]$ initializes its local model as the current global model, i.e., $\model^{(\globalIteration,0)}_i=\model^{(\globalIteration)}$, and updates it for $\TotalLocalIter \geq 1$ local iterations as
\[
    \model^{(\globalIteration,\localIteration+1)}_i = \model^{(\globalIteration,\localIteration)}_i-\learningRate_i \cdot \nabla F(\model_i^{(\globalIteration,\localIteration)};\tilde{\Dataset}_i), 
\]
where $\learningRate_i$ is the local learning rate, $\localIteration$, $0 \leq \localIteration \leq \TotalLocalIter -1$, is the local iteration index, and $\nabla F(\model_i^{(\globalIteration,\localIteration)};\tilde{\Dataset}_i)$ is the gradient of the loss function evaluated at a subset of user $i$'s data $\tilde{\Dataset}_i\subseteq\Dataset_i$. Upon finishing local training, users return their model updates $
    \modelupdate_i^{(\globalIteration)}=\model_i^{(\globalIteration,\TotalLocalIter)}-\model^{(\globalIteration)}
$ 
to the federator. 

Meanwhile, the federator runs the same computations on $\Dataset_0$ to obtain a federator model update $\modelupdate_0^{(\globalIteration)}$.
Upon receiving the local updates, the federator aggregates them according to some aggregation rule $\AGG$, i.e., it computes $\aggresult^{(\globalIteration)} = \AGG(\modelupdate_0^{(\globalIteration)}, \cdots, \modelupdate_n^{(\globalIteration)})$. The federator obtains the global update $\modelupdate^{(\globalIteration)}=\aggresult^{(\globalIteration)}$. Consider $\learningRate$ as the global learning rate, the federator updates the global model\footnote{
While presented in the simple gradient descent setting, our scheme does not depend on the exact update rule and is applicable to, e.g., momentum, higher order methods, and adaptive learning rate schedules.}
\[ %
    \model^{(\globalIteration+1)} = \model^{(\globalIteration)}-\learningRate \cdot \modelupdate^{(\globalIteration)}.
\] %

\subsection{Threat Model and Defense Goals}
We denote by $\colset \subset [\TotalUser]$, $|\colset| = \Colluding$, the set of colluding users, by $\byzset \subset [\TotalUser]$, $|\byzset| = \Byzantine$, the set of Byzantine users and by $\honestset \define [n] \setminus \colset \cup \byzset$ the set of benign users. 

\subsubsection{IT Privacy} %
We do not assume curious entities to have limited computing resources. Our schemes guarantee IT privacy of the benign users' local updates in each iteration.

We consider a set of $\colset$ colluding users. The federator is assumed to be honest-but-curious, i.e., it tries to infer private data about the benign users, but it honestly conducts the protocol, and does not proactively leak data that would compromise the system or give an advantage to adversaries. 

Since we allow the federator to collude with the set $\colset$ of colluding users. We require that the colluding users and the federator cannot learn further information about the benign users' local model updates $\modelupdate_{\honestset}^{(\globalIteration)}$ beyond what can inherently be inferred from
\begin{enumerate*}[label=\emph{(\alph*)}]
    \item the colluding users' local model updates $\modelupdate_{\mathcal{T}}^{(\globalIteration)}$ and the federator model update $\modelupdate_0^{(\globalIteration)}$;  \item the colluding users' local datasets $\Dataset_{\mathcal{T}}$ and the the root dataset $\Dataset_0$; 
    \item the current global model $\model^{(\globalIteration)}$; and 
    \item the current aggregation $\aggresult^{(\globalIteration)}$;  
\end{enumerate*}
all of which are required by the learning algorithm and may give information about the other users' local datasets.
For global iteration $\globalIteration$, the privacy guarantee is formulated as:
\begin{equation}\label{eq:privacy_combine}
    I(\modelupdate_{\honestset}^{(\globalIteration)}; M_{\mathcal{T}}^{(\globalIteration)}, M_0^{(\globalIteration)} \!\mid \!\modelupdate_{\mathcal{T}}^{(\globalIteration)}, \modelupdate_0^{(\globalIteration)}, \Dataset_{\mathcal{T}}, \Dataset_0, \model^{(\globalIteration)}, \aggresult^{(\globalIteration)})\!=\!0, 
\end{equation}
where $M_{\mathcal{T}}^{(\globalIteration)}$ is the set of potential messages\footnote{We will consider interactive algorithms, where the clients and the federator exchange messages to mitigate the effect of Byzantine users.} received by the colluding users during the execution of the algorithm at iteration $\globalIteration$, and $M_0^{(\globalIteration)}$ denotes the messages received by the federator during the execution of the algorithm at iteration $\globalIteration$. 
We omit the global iteration index $\globalIteration$ in the rest for brevity.

\subsubsection{Byzantine Resilience} We consider Byzantine users arbitrarily deviating from the protocol and having access to all users' datasets. Specifically, Byzantine users may: \begin{enumerate*}[label=\emph{(\alph*)}]
    \item corrupt their local updates; and \item perform the secret sharing and all other computations dishonestly. 
\end{enumerate*}  A federated learning scheme is said to be Byzantine resilient if it converges, i.e., solves the optimization problem in \cref{eq:optimization}, even if up to $\Byzantine$ Byzantine users behave maliciously.

\subsubsection{Dropout Tolerance} Users can drop out during the protocol. This may be due to users experiencing delays, network congestion or similar reasons. A dropout tolerant federated learning scheme converges even if up to $\Dropout$ users drop out per iteration of the algorithm.

\subsubsection{Desired Schemes} Our goal is to design a FL scheme that simultaneously satisfies the following: \begin{enumerate*}[label=\emph{(\alph*)}]
    \item IT privacy against an honest-but-curious federator and against any collusion of up to $\Colluding$ users; \item Byzantine resilient against $\Byzantine$ Byzantine users; and \item dropout tolerance of up to $\Dropout$ users.
\end{enumerate*}

\subsection{Preliminaries} 
\subsubsection{Threshold Secret Sharing}
We rely on the powerful tool of $(\TotalUser,\SSthreshold,\Colluding)$-\emph{threshold secret sharing}. Consider a secret vector $\bm{s}$ drawn from a finite alphabet, usually a finite field $\mathbb{F}_p^\dimension$, partitioned into $\partition=\SSthreshold-\Colluding$ sub-vectors\footnote{We assume $\partition \mid \dimension$. Otherwise, this can be ensured through zero-padding.}, $\bm{s}_i\in \mathbb{F}_p^{\dimension/\partition}$, i.e., $\bm{s}^T=[\bm{s}_1^T,\cdots,\bm{s}_\partition^T]$. An $(\TotalUser,\SSthreshold,\Colluding)$-threshold secret sharing (SecShare) scheme encodes $\bm{s}$ into $\TotalUser$ secret shares denoted by $\bm{s}[i]$ for $i \in [n]$. The encoding utilizes a degree-$(\SSthreshold-1)$ polynomial $f(x)$ encoding the $\partition$ sub-vectors with $\Colluding$ vectors drawn independently and uniformly at random from the finite field $\mathbb{F}_p^{\dimension/\partition}$. An example of such a polynomial is given in the sequel. Each secret share is an evaluation of the encoding polynomial at a distinct non-zero evaluation point $\alpha_i \in \mathbb{F}_p$, i.e., $\bm{s}[i]=f(\alpha_i)$ for $i\in [\TotalUser]$. Any $\SSthreshold\leq\TotalUser$ or more shares allow full reconstruction of $\bm{s}$, while any $\Colluding<\SSthreshold$ or less shares are statistically independent of $\bm{s}$. We call $\SSthreshold$ the threshold of the SecShare scheme. Instantiations of a $(\TotalUser,\SSthreshold,\Colluding)$-threshold SecShare scheme can be Shamir SecShare~\cite{shamir1979share}, McEliece-Sarwate SecShare~\cite{mceliece1981sharing}, Lagrange coded computing~\cite{yu2019lagrange} or additive SecShare (see e.g.~\cite{bendlin2011semi}). 

The SecShare schemes we consider in this paper are linear, i.e., homomorphic with addition and scaling by a constant, however they are not inherently homomorphic with multiplication.
Thus, additional steps are needed for multiplication on the secret shares, i.e., either by re-randomizing the coefficients of the product of the polynomials and requiring more users for decoding~\cite{gennaro1998simplified,asharov2017full}, or by unlocking the multiplicative homomorphism through converting multiplication to linear computations using Beaver triples~\cite{beaver1992efficient}.

In our setting, users secret-share some information with each other that will be further processed. Verifying that all the shares come from the same encoding can be done through the use of cryptographic commitments to the coefficients of the encoding polynomial, e.g.,~\cite{so2020byzantine,jahani2023byzantine}. However, this verification step weakens the privacy to computational privacy guarantees. To guarantee IT privacy, IT verifiable secret sharing (ITVSS)~\cite{BGW,asharov2017full} is needed. ITVSS uses a bivariate polynomial $\bm{S}(x,y)$ to share the secret, where each user receives two univariate polynomials $f^i(x)=\bm{S}(x,\alpha_i)$ and $\bm{g}^i(x)=\bm{S}(\alpha_i,y)$ as its share, where $f(x)$ will be the encoding polynomial of the actual secret sharing and $\bm{g}(x)$ will be the authentication information to guarantee the use of the identical encoding polynomial $f(x)$. The secret share is a valid share if $f^i(\alpha_j)=\bm{S}(\alpha_j,\alpha_i)=\bm{g}^j(\alpha_i)$, which can be verified with the help of all other users~\cite{BGW,asharov2017full}.

\subsubsection{Lagrange Coded Computing}\label{sec:LCC}
Lagrange Coded Computing (LCC)~\cite{yu2019lagrange} is one instantiation of the threshold SecShare schemes. 
In LCC, a secret vector $\bm{s}$ is partitioned and secret-shared among $\TotalUser$ users using the degree-$(\partition+\Colluding-1)$ polynomial
\begin{equation}
    \begin{aligned}
        f_{\bm{s}}{(x)} & = \sum_{j \in [\partition]}\bm{s}_j\cdot \prod_{l \in [\partition+t]\setminus \{j\}} \frac{x-\beta_l}{\beta_j-\beta_l} \\
        & + \sum_{j \in [\Colluding]}\bm{r}_{j} \cdot \prod_{l \in [\partition+t] \setminus \{\partition+j \}} \frac{x-\beta_l}{\beta_{\partition+j}-\beta_l}, \forall j \in [\TotalUser],
    \end{aligned}
\end{equation}
where $\beta_1, \cdots, \beta_{\partition+t}$ are $\partition+t$ distinct non-zero elements from $\mathbb{F}_p$ and the $\bm{r}_{j}$'s are chosen independently and uniformly at random from $\mathbb{F}_p^{\dimension/\partition}$. Note that $f_{\bm{s}}(\beta_1)=\bm{s}_{1}, \cdots, f_{\bm{s}}(\beta_{\partition})=\bm{s}_{\partition}$. Secret shares are computed by evaluating $f_{\bm{s}}(x)$ at $\TotalUser$ distinct non-zero values $\{ \alpha_l \}_{l \in [\TotalUser]}$, which are selected from $\mathbb{F}_p$ such that $\{ \alpha_l \}_{l \in [\TotalUser]} \cap \{ \beta_l \}_{l \in [\partition]}=\varnothing$, i.e. $\bm{s}[i]=f_{\bm{s}}(\alpha_i)$.
The reconstruction of $\bm{s}$ follows by interpolating $f_{\bm{s}}(x)$ and evaluating it at $x = \beta_1, \cdots, \beta_m$.

It is worth mentioning that for an arbitrary degree-$\approdegree$ polynomial $\discripoly$, LCC allows the computation of $\discripoly(\bm{s})$ performed over the secret shares of $\bm{s}$. Specifically, each user $i \in [\TotalUser]$, holding $\bm{s}[i]=f_{\bm{s}}(\alpha_i)$, computes $\discripoly(\bm{s}[i])=h(f_{\bm{s}}(\alpha_i))$ locally. This gives an evaluation of the polynomial $\discripoly(f_{\bm{s}}(x))$ at the point $\alpha_i$. Upon having more than $(m+t-1) \approdegree + 1$ correct evaluations from the users, $\discripoly(f_{\bm{s}}(x))$ can be interpolated. Obtaining $\discripoly(\bm{s})$ is done by evaluating $\discripoly(f_{\bm{s}}(x))$ at $\{ \beta_l \}_{l \in [\partition]}$, i.e., $\discripoly(\bm{s})=[\discripoly(\bm{s}_1)^T,\cdots, \discripoly(\bm{s}_{\partition})^T]^T=[\discripoly(f_{\bm{s}}(\beta_1))^T,\cdots,\discripoly(f_{\bm{s}}(\beta_\partition))^T]^T$.

\subsubsection{Beaver Triples}
To unlock the multiplicative homomorphism of threshold SecShare schemes, we use \emph{Beaver triples}. A {Beaver triple} consists of two random variables $\gamma$ and $\omega$ drawn independently and uniformly at random from a finite field $\mathbb{F}_p$ and their product $\kappa=\gamma \omega$. The values of $\gamma, \omega$ and $\kappa$ are secret-shared independently to $\TotalUser$ users using an $(\TotalUser,\SSthreshold, \Colluding)$-threshold SecShare scheme.

Assume we have two secrets $s$ and $v$ and a Beaver triple that are secret-shared with $\TotalUser$ users using an identical threshold SecShare scheme. Each user $i$ holding the shares of the Beaver triple $\gamma[i]$, $\omega[i]$, $\kappa[i]$, and the shares of the secrets $s[i]$ and $v[i]$, wants to compute a secret share of the product $sv$ without knowing the values of $s$ nor $v$. 
To that end, each user: 
\begin{enumerate*}[label=\emph{(\alph*)}]
    \item computes $s[i]-\gamma[i]$ and $v[i]-\omega[i]$;
    \item sends the results to one dedicated user (in our case, the federator), who reconstructs $s-\gamma$ and $v-\omega$ and publicly announces their values; and
    \item computes its secret share of the product as 
\end{enumerate*}
$
(sv)[i]\define(s-\gamma)(v-\omega)+(s-\gamma)\omega[i]+(v-\omega)\gamma[i]+\kappa[i],
$
which only involves addition and scaling of secret shares. 
Reconstructing from enough shares $(sv)[i]$ gives $sv\!=\!(s-\gamma)(v-\omega)+(s-\gamma)\omega+(v-\omega)\gamma+\kappa$. 

Therefore, using Beaver triples, we convert the multiplication on the secret shares to linear computations, which can be trivially solved due to the additive homomorphism of the chosen threshold SecShare scheme.

\section{Main Results}
We introduce \ourscheme and \ourschemelo, two Byzantine-resilient, dropout tolerant, and IT private FL schemes. The former has a high communication cost incurred by the joint privacy and Byzantine resilience requirements. The latter uses a trusted third party only during an initialization phase to reduce communication costs.
Each training iteration of \ourscheme and \ourschemelo consists of the following main steps:
\begin{enumerate}[label=\Alph*.]
    \item Users \textbf{compute, normalize and quantize} their local \textbf{model updates}. Similarly, the federator computes, normalizes and quantizes the federatot model update.
    \item Each user's \textbf{update} is \textbf{secret-shared} among all users. 
    \item \textbf{Validation of the normalization} based on the received secret shares.
    \item Users \textbf{compute a secret representation} of the aggregation of the validated updates.
    \item The federator receives shares of the aggregation from the users to \textbf{reconstruct the aggregated updates}.
\end{enumerate}

The main steps are illustrated in \cref{fig:training}. To guarantee Byzantine resilience, our schemes utilize ideas from FLTrust \cite{cao2020fltrust}. Specifically, a TS is computed for each user using a discriminator function. Users' local updates are scaled by their TSs during aggregation. To simultaneously enable IT privacy, we design a new discriminator function based on polynomials that can be computed on the secret shares. The main challenge is designing the discriminator function such that it both \begin{enumerate*}[label=(\emph{\alph*})]
    \item is computable over secret shares in a private manner; and
    \item prevents the malicious gradients from affecting the convergence of the algorithm.
\end{enumerate*}
Next, we provide the main ingredients and theoretical guarantees for each of our schemes.

\begin{figure}[!t]
    \vspace{0.2cm}
    \centering
    \hspace*{-0.8cm} %
    \begin{adjustbox}{center}
        \resizebox{0.5\textwidth}{!}{\input{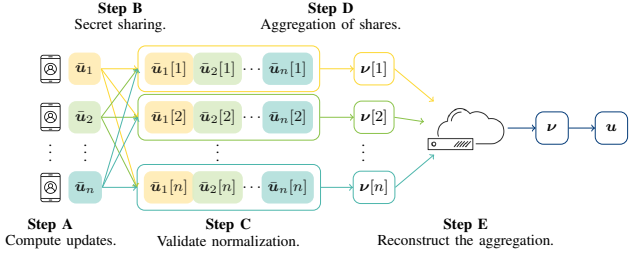}}
    \end{adjustbox}
    \caption{The main steps in each training iteration of \ourscheme and \ourschemelo.}
    \vspace{-0.3cm}
    \label{fig:training}
\end{figure}

\subsection{\ourscheme}
Each local update vector is partitioned into $\partition$ sub-vectors and secret-shared among all users using LCC~\cite{yu2019lagrange}, which is a threshold SecShare scheme and is described in \cref{sec:LCC}. The shares are checked for corruptions using an ITVSS scheme~\cite{BGW, asharov2017full}. Users compute the aggregation on the secret shares using a carefully designed discriminator function, and re-randomize~\cite{gennaro1998simplified} before sending the result to the federator to maintain perfect privacy. Treating the erroneous computations sent by Byzantine users as errors and the dropouts as erasures in a Reed-Solomon (RS) code, cf. \cite{mceliece1981sharing}, the federator decodes the error-free aggregation result. 

\begin{thm}\label{thm:theorem_by} 
Consider a federated learning setting with $\TotalUser$ users, out of which $\Byzantine$ are Byzantine, $\Colluding$ are curious and $\Dropout$ may drop out at any iteration with $\TotalUser \geq 2\Byzantine+(\approdegree+2)\cdot (\partition+\Colluding-1)+\Dropout+1$, where $\partition$ is the number of sub-vectors into which each local update is partitioned and $\approdegree$ is the degree of the discriminator polynomial. \ourscheme  guarantees the following:
\begin{enumerate}[label=\arabic*), leftmargin=1.5em, itemsep=0pt, parsep=0pt, topsep=0pt]
        \item IT privacy against any $\Colluding$ users and the federator according to \cref{eq:privacy_combine} in each iteration;
        \item Resilience against $\Byzantine$ Byzantine users and robustness against $\Dropout$ dropouts;
        \item Communication cost of $O(\frac{\dimension}{\partition}n^3+n^4)$ scalars per user and $O(\frac{d}{m}n+n^2)$ for the federator, per iteration. Computation cost of $O\left((\frac{d}{\partition}n^3+n^4)\log^2n\log\log n\right)$ per user and
        $O\left((\frac{d}{m}n+n^2)\log^2n\log\log n\right)$ at the federator; and
        \item Convergence as shown in \cref{thm:convergence}.
    \end{enumerate}
    
\end{thm}

\begin{proof}
    The proof is given in \cref{app:ByITFL}.
\end{proof}

\begin{table*}[tb] %
\vspace{.15cm}
\caption{Complexity Analysis with respect to the total number of users $\TotalUser$, the dimension of the model updates $\dimension$, the partitioning parameter $\partition$, and the degree $\approdegree$ of the discriminator polynomial. \vspace{-.1cm}}\label{tab:complexity}
\centering
\begin{tabular}{c|cc|cc}
            & \multicolumn{2}{c|}{Per-User}   & \multicolumn{2}{c}{Federator} \\ \hline
            & Computation & Communication & Computation & Communication \\ \hline
        BREA~\cite{so2020byzantine} & $O(\dimension \TotalUser\log^2\TotalUser+\dimension\TotalUser^2)$ &  $O(\dimension \TotalUser+\TotalUser^2)$ & $O((\dimension \TotalUser+\TotalUser^3)\log^2\TotalUser\log\log \TotalUser)$ & $O(\dimension \TotalUser+\TotalUser^3)$ \\ \hline 
        ByzSecAgg~\cite{jahani2023byzantine} & $O(\frac{\dimension}{\partition}\TotalUser\log^2\TotalUser+\frac{\dimension}{\partition}\TotalUser^2)$ &  $O(\frac{\dimension}{\partition}\TotalUser+\TotalUser^2)$ & $O((\frac{\dimension}{\partition}\TotalUser+\TotalUser^3)\log^2\TotalUser\log\log \TotalUser)$ & $O(\frac{\dimension}{\partition}\TotalUser+\TotalUser^3)$  \\ \hline
        ByITFL & $O((\frac{\dimension}{\partition}\TotalUser^3+\TotalUser^4)\log^2n\log\log \TotalUser)$ &  $O(\frac{\dimension}{\partition}\TotalUser^3+\TotalUser^4)$ & $O((\frac{\dimension}{\partition}\TotalUser+\TotalUser^2)\log^2\TotalUser\log\log \TotalUser)$ & $O(\frac{\dimension}{\partition}\TotalUser+\TotalUser^2)$ \\ \hline
        LoByITFL & $O((\frac{\dimension}{\partition}+\approdegree)\TotalUser)$ &  $O((\frac{\dimension}{\partition}+\approdegree)\TotalUser)$ & $O((\frac{\dimension}{\partition}+\approdegree)\TotalUser^2\log^2\TotalUser\log\log \TotalUser)$ & $O((\frac{\dimension}{\partition}+\approdegree)\TotalUser^2)$ 
        \end{tabular}
\vspace{-0.3cm}
\end{table*}

\subsection{\ourschemelo}
We introduce \ourschemelo to reduce the high communication and computation costs incurred by \ourscheme, while still ensuring Byzantine resilience and IT privacy. This reduction comes at the expense of requiring a Trusted Third Party (TTP), which, however, is only used in an initialization phase to distribute Beaver triples~\cite{beaver1992efficient} to the users and the federator. 
Beaver triples enable a multiplicative homomorphism of SecShare schemes, allowing polynomial computations on shared updates without re-randomization and without increasing the degree of the encoding polynomial of the secret sharing. 
Additionally, an additively homomorphic Message Authentication Code (MAC) is utilized to ensure integrity, preventing malicious users from performing corrupt polynomial computation~\cite{bendlin2011semi}.

\begin{thm}
    \label{thm:theorem_loby}
    Consider a federated learning setting with a TTP and $\TotalUser$ users, out of which $\Byzantine$ are Byzantine, $\Colluding$ are curious, and $\Dropout$ may drop out at any iteration with $\TotalUser \geq \Byzantine+\partition+\Colluding+\Dropout$, where $\partition$ is the number of sub-vectors into which each local update vector is partitioned.
    \ourschemelo guarantees the following:
    \begin{enumerate}[label=\arabic*), leftmargin=1.5em, itemsep=0pt, parsep=0pt, topsep=0pt]
        \item IT privacy against any $\Colluding$ users and the federator according to \cref{eq:privacy_combine} in each iteration;
        \item Resilience against $\Byzantine$ Byzantine users and robustness against $\Dropout$ dropouts;
        \item Communication cost of $O\left((\frac{\dimension}{\partition}+\approdegree)\TotalUser\right)$ scalars per user and $O\left((\frac{\dimension}{\partition}+\approdegree)\TotalUser^2\right)$ for the federator, per iteration; Computation cost of $O\left((\frac{\dimension}{\partition}+\approdegree)\TotalUser^2\log^2\TotalUser\log\log \TotalUser\right)$ per user and $O\left((\frac{\dimension}{\partition}+\approdegree)\TotalUser\right)$ operations at the federator; and
        \item Convergence as shown in \cref{thm:convergence}.
    \end{enumerate}
    
\end{thm} %

\begin{proof}
    The proof is given in \cref{app:LoByITFL}.
\end{proof}

\subsection{Convergence for \ourscheme and \ourschemelo}
We denote by $F(\model)=\mathbb{E}[F(\model,\Dataset)]$ the expectation of the loss on the union of the users' datasets. We make the same assumptions as in~\cite[Assumptions 1, 2 and 3]{cao2020fltrust}:
\begin{enumerate*}[label=(\emph{\alph*})]
    \item The expected loss $F(\model)$ is $\mu$-strongly convex and differentiable with $L$-Lipschitz continuity, and the loss function $F(\model, \Dataset_i)$ is $L_1$-Lipschitz probabilistically with parameter $\delta$.
    \item The gradient $\nabla F(\model, \Dataset_i)$ of the loss function at the optimal global model, and the gradient difference between the loss with the optimal global model and the loss with any $\model$ are bounded. Specifically certain inner products with an arbitrary unit vector are sub-exponential with parameters $\sigma_1$ and $\Lambda_1$ in the former case, and $\sigma_2$ and $\Lambda_2$ in the latter case.
    \item Each user's local dataset $\Dataset_i$ for $i \in \TotalUser$ and the root dataset $\Dataset_0$ are sampled independently from the same distribution.
\end{enumerate*}

\begin{thm}\label{thm:convergence}
    Given the assumptions above, \ourscheme and \ourschemelo converge, i.e., the difference between the global model $\model^{(\globalIteration)}$ obtained by the federator after $\globalIteration$ global iterations and the optimal global model $\model^\star\define\argmin F(\model)$ obtained in the case where all users would have been honest is bounded by %
    \begin{equation*}
        \| \model^{(\globalIteration)} - \model^\star \| \leq (1-\Gamma)^\globalIteration \| \model^{(0)}-\model^\star\|+12\learningRate \Delta_1/\Gamma,
    \end{equation*}
    where $\learningRate$ is the global learning rate, 
    \begin{align*}
        \Gamma \,\; & = 1-(\sqrt{1-\mu^2/4L^2}+24\learningRate\Delta_2+2\learningRate L),\\
        \Delta_1 & = \sigma_1\sqrt{\frac{2}{|D_0|}\left(\dimension\log6+\log(\frac{3}{\delta})\right)},\\
        \Delta_2 & =\!\! \sqrt{\frac{2 \sigma_2^2}{|D_0|}\left(\! \dimension\log\frac{18L_2}{\sigma_2} \!+ \! \frac{\dimension}{2}\log\frac{|D_0|}{\dimension} \! +\! \log\frac{6\sigma_2^2 z \sqrt{|D_0|}}{\Lambda_2\sigma_1\delta}\right)},
    \end{align*}%
    $|D_0|$ is the size of the root dataset, $L_2\!=\!\max \{ L, L_1 \}$ and $z$ is a positive integer satisfying $\| \model-\model^*\| \leq z \sqrt{\dimension}$ for all $\model \in \mathbb{R}^\dimension$.
\end{thm}

\begin{proof}
    The proof is given in \cref{proof:convergence}.
\end{proof}

In \cref{sec:ByITFL} and \cref{sec:LoByITFL}, we present \ourscheme and \ourschemelo in detail, respectively. We compare the communication and computation complexity of \ourscheme with respect to $\TotalUser$, $\dimension$ and $\partition$ to BREA~\cite{so2020byzantine} and ByzSecAgg~\cite{jahani2023byzantine} in \cref{tab:complexity}.

\section{\ourscheme}\label{sec:ByITFL}
We explain the main steps of \ourscheme followed in each global iteration. In this description, we omit the current global iteration index $\globalIteration$ for clarity of presentation.

\subsection{Normalization and Quantization}
\label{subsec: normal}
To defend against Byzantine attacks performed on the magnitude of the local update vector, the federator, and all users first normalize $\modelupdate_i$ to a unit vector $\realnormalmodelupdate_i$,
\[ %
    \realnormalmodelupdate_i \define \frac{\modelupdate_i}{\| \modelupdate_i \|},  \forall i \in \{0,1,\cdots,\TotalUser\}.
\] %
This prevents Byzantine users from influencing the model by sending extremely large/small local updates.

Since the training process is performed in the real domain and LCC (like every IT private SecShare) works over finite fields, it is essential to transfer $\realnormalmodelupdate_i \in \mathbb{R}^{\dimension}$ to vectors in a finite field $\normalmodelupdate_i \in \mathbb{F}_p^\dimension$, where $p$ is a large prime or power of a prime. Therefore, users apply an element-wise stochastic quantizer $Q_q(x)$ with $2q+1$ quantization intervals as in \cite{so2020byzantine}, \cite{jahani2023byzantine}:
\begin{equation*}
    Q_q(x) = 
    \left\{
    \begin{array}{ll}
          \frac{\lfloor qx \rfloor}{q}   & \text{with prob. } 1-(qx-\lfloor qx \rfloor), \\
          \frac{\lfloor qx \rfloor+1}{q} & \text{with prob. } qx-\lfloor qx \rfloor,
    \end{array} 
    \right. 
\end{equation*} 
where $q$ is the number of quantization levels. The relation between $p$ and $q$ is explained later. Note that the stochastic quantization is unbiased, i.e., $\mathbb{E}_Q[Q_q(x)]=x$. Let $\Phi(x) \define x \mod p$ be the function mapping integers to values in $\mathbb{F}_p$. The quantization is defined as
$
    \normalmodelupdate_i \define \Phi(q\cdot Q_q(\realnormalmodelupdate_i)).
$

\subsection{Sharing of the Normalized Model Updates}
The normalized update vectors are partitioned into smaller sub-vectors and secret-shared using LCC and ITVSS. The partition of $\normalmodelupdate_i$ into $\partition$ smaller subvectors is done as
\begin{equation*}
    \normalmodelupdate_i = [\normalmodelupdate_{i1}^T,\normalmodelupdate_{i2}^T,\cdots,\normalmodelupdate_{i\partition}^T]^T, \forall i \in \{ 0,1,\cdots,\TotalUser\},
\end{equation*}
where each sub-vector is of size $\frac{\dimension}{\partition}$ and $\partition \leq \frac{\TotalUser-\Dropout-1}{\approdegree+2}-\Byzantine-\Colluding+1$.

We assume $\modelupdate_0$ does not require privacy. The federator broadcasts $\normalmodelupdate_0$ to the users. Each user $i$ secret-shares $\normalmodelupdate_i$ with all users by LCC with the degree-$(\partition+\Colluding-1)$ encoding polynomial
\begin{equation}
    \begin{aligned}
        f_i(x) & = \sum_{j \in [\partition]}\normalmodelupdate_{ij} \cdot \prod_{l \in [\partition+\Colluding]\setminus \{j\}} \frac{z-\beta_l}{\beta_j-\beta_l} \\
        & + \sum_{j \in [\Colluding]}\bm{r}_{ij} \cdot \prod_{l \in [\partition+t] \setminus \{\partition+j \}} \frac{z-\beta_l}{\beta_{\partition+j}-\beta_l}, \forall i \in [\TotalUser],
    \end{aligned}
\end{equation}
where $\beta_1, \cdots, \beta_{\partition+t}$ are $\partition+t$ distinct non-zero elements from $\mathbb{F}_p$ and $\bm{r}_{ij}$'s are chosen independently and uniformly at random from $\mathbb{F}_p^{\dimension/\partition}$. Note that the finite field size $p$ should be large enough to avoid any wrap-around as we describe in \cref{sub:aggregation}. Each user $j \in [n]$ receives a secret share of $\normalmodelupdate_i$ from user $i \in [n]$, i.e., $\normalmodelupdate_{i}[j]=f_i(\alpha_j)$ where $\alpha_1, \cdots, \alpha_{\TotalUser}$ are distinct non-zero element from $\mathbb{F}_p$ that satisfy $\{ \alpha_l \}_{l \in [\TotalUser]} \cap \{ \beta_l \}_{l \in [\partition]}=\varnothing$.
We leverage the ITVSS protocol from \cite{BGW} to prevent Byzantine users from sending corrupt shares in the secret sharing step.

\subsection{Validation of Normalization}\label{sub:normalization}
Byzantine users may misbehave during the normalization. 
Thus, upon receiving a secret share, user $i \in [\TotalUser]$ computes $\langle \normalmodelupdate_j[i], \normalmodelupdate_j[i] \rangle$ for each $j\in [\TotalUser]$. Due to the linearity of secret sharing, this dot-product also corresponds to a secret share of the norm of $\normalmodelupdate_j$, i.e.,
     $\| \normalmodelupdate_j \|_2^2 [i]=  \langle \normalmodelupdate_j[i], \normalmodelupdate_j[i] \rangle$.

If each user $i$ sends the share $\| \normalmodelupdate_j \|_2^2 [i]$ to the federator, the latter can decode $\| \normalmodelupdate_j \|_2^2$ and verify whether it is equal to one. However, the encoding polynomial of those shares $\| \normalmodelupdate_j \|_2^2 [i]$ includes additional information about $\normalmodelupdate_j$. Therefore, simply sending the shares violates the IT privacy requirement. To that end, the users \emph{re-randomize}~\cite{gennaro1998simplified,asharov2017full} the share $\| \normalmodelupdate_j \|_2^2 [i]$ before sending it to the federator, which involves sub-sharing the users' secret shares using ITVSS \cite{BGW}, and linearly combining to construct the re-randomized secret shares. %
Upon receiving enough re-randomized shares $\| \normalmodelupdate_j \|_2^2 [i]$, the federator utilizes error correction decoding of the underlying RS code to reconstruct $\| \normalmodelupdate_j \|_2^2$ for each $j \in [\TotalUser]$ and checks if it is within a certain interval, i.e.,
\begin{align*}
    \left |   \| \normalmodelupdate_j \|_2^2 - \Phi(q\cdot Q_q(1))^2 \right | < \varepsilon \cdot q^2,
\end{align*}
where $\varepsilon$ is a predefined threshold and can be set empirically. %
The interval is caused by the accuracy loss due to quantization. If any user does not pass the normalization check, the federator marks them as Byzantine and excludes them from future computations. Note that decoding the RS code for error correction requires $\TotalUser \geq 2\Byzantine+2(\partition+\Colluding-1)+\Dropout+1$. 

\subsection{Users Secure Computation}\label{sub:computation}
Users compute the discriminator function on the secret shares of the model updates, and construct secret shares of the aggregation result. 

This step is based on the non-private scheme FLTrust~\cite{cao2020fltrust}, which assigns to each user $i$ a trust score determined by the cosine similarity between the federator and the local model update, i.e., $\operatorname{TS}_i = \operatorname{ReLU}\left(\cos(\theta_i)\right)$, where $\theta_i$ is the angle between $\modelupdate_i$ and $\modelupdate_0$. $\operatorname{ReLU}$ is the discriminator function in FLTrust. The federator then scales the local model updates by their trust scores and averages them for aggregation.

Making FLTrust IT private is not straightforward, which is why we replace ReLU by a degree-$\approdegree$ polynomial $\discripoly(x)=h_0+h_1x+ \cdots + h_kx^{\approdegree}$ as the discriminator function. The choice of the discriminator function will be discussed later in \cref{sec:discriminator}.
Therefore, the trust score for each user becomes
$
    \trustscore_i = \discripoly(\cos(\theta_i))=\discripoly(\langle \normalmodelupdate_0,\normalmodelupdate_i \rangle), \forall i \in [\TotalUser], 
$
and the aggregation is
\begin{equation}\label{eq:aggre}
\begin{aligned}
    \aggresult & =\frac{1}{\sum_{i\in [\TotalUser]}{\trustscore_i}} \cdot \sum_{i\in [\TotalUser]}{(\trustscore_i \cdot \normalmodelupdate_i)}
    \define \frac{\sumtwo}{\sumone}, 
\end{aligned}
\end{equation}
\begin{equation*}
\begin{aligned}
    \text{with } \sumone & \!\define\! \sum_{i\in [\TotalUser]}{\discripoly(\langle \normalmodelupdate_0,\normalmodelupdate_i \rangle)} \text{ and }
    \sumtwo \!\define\! \sum_{i\in [\TotalUser]}{\left(\discripoly(\langle \normalmodelupdate_0,\normalmodelupdate_i \rangle) \cdot \normalmodelupdate_i\right)}.
\end{aligned}
\end{equation*}

The federator needs to compute $\aggresult$ in a private manner without learning individual users' private information beyond this quotient. The colluding users should learn nothing about other honest users' data during the computation. Both $\sumone$ and $\sumtwo$ are polynomial functions of the model updates $\normalmodelupdate_0$ and $\normalmodelupdate_i$ for $i \in [\TotalUser]$, where $\normalmodelupdate_0$ is known to all users and $\normalmodelupdate_i$'s are secret-shared among the users using LCC. Since LCC allows the computation of an arbitrary polynomial $h$ over its secret, as described in \cref{sec:LCC}, users compute the polynomials $\sumone$ and $\sumtwo$ on their secret shares, such that each of them obtains an evaluation of $\sumone$ and $\sumtwo$. This guarantees that any set of up to $\Colluding$ users are not able to learn anything from the shares. 

Privacy against the federator has not yet been guaranteed: if the federator were to reconstruct $\sumone$ and $\sumtwo$ directly, some information about $\normalmodelupdate_1, \cdots, \normalmodelupdate_\TotalUser$ would leak to the federator. 
LCC, like other threshold SecShare schemes, is additively homomorphic, but not multiplicatively.
We follow the re-randomization from \cite{gennaro1998simplified,asharov2017full} to construct the re-randomized secret shares. 

The users now hold the re-randomized shares $\sumone[i]$ and $\sumtwo[i]$. 
It remains to ensure that the federator obtains the quotient $\sumtwo/\sumone$ without gaining any additional information about $\sumone$ and $\sumtwo$. 
To this end, each user $i$:
\begin{enumerate*}[label=\emph{(\alph*)}] 
\item chooses an independent value $\random_i$ uniformly at random from $\mathbb{F}_p$ and secret-shares it by LCC and ITVSS among all users; 
\item adds the shares of $\random_j$'s from all user $j$ and obtains its share of\footnote{The case $\random=0$ can be avoided by minor changes, omitted for brevity.}  $\random = \sum_{j=[\TotalUser]}{\random_j}$; and
\item multiplies $\sumone[i]$ and $\sumtwo[i]$ by $\random[i]$ and performs re-randomization to obtain $\random \sumone[i]$ and $\random \sumtwo[i]$.
\end{enumerate*} 
Users send the secret shares $\random \sumone[i]$ and $\random \sumtwo[i]$ to the federator.

\subsection{Private Aggregation}\label{sub:aggregation}
The federator receives secret shares of the aggregation from the users to reconstruct the private aggregation by decoding an error correcting code and updates the global model. 

The federator receives $(\random \sumone)[i]$ and $(\random \sumtwo)[i]$, for which the degree of the encoding polynomial is $(\approdegree+1)(\partition+\Colluding-1)$ and $(\approdegree+2)(\partition+\Colluding-1)$, respectively.
With sufficient number of users returning evaluations, the federator is able to leverage the error correction capability of RS codes \cite{mceliece1981sharing} to decode $\random \sumone$ and $\random \sumtwo$. Therefore, we require $\TotalUser \geq 2\Byzantine+(\approdegree+2)\cdot (\partition+\Colluding-1)+\Dropout+1$. 

Upon decoding the correct values, the federator computes $\aggresult\!=\!\frac{\random \sumtwo}{\random \sumone}$,  converts it from the finite field back to the real domain through de-quantizing by $Q_q(x)^{-1}$ and demapping by $\Phi^{-1}$. The federator then checks whether $\aggresult$ and $\modelupdate_0$ point to the same general direction, computes the model update
\begin{align}\label{eq:modelupdate}
    \modelupdate =
    \left\{
    \begin{array}{ll}
          \| \modelupdate_0\| \cdot \frac{\aggresult}{\| \aggresult\|},  & \text{if } \langle \aggresult,\modelupdate_0 \rangle  \geq 0, \\ 
          - \| \modelupdate_0\| \cdot \frac{\aggresult}{\| \aggresult\|},  & \text{if } \langle \aggresult,\modelupdate_0 \rangle  < 0,
    \end{array} 
    \right. 
\end{align}
and updates the global model for the next iteration.
To ensure the correctness of the result, none of the computations should cause a wrap-around in the finite field. Each entry of the normalized gradient is in the range $-q$ to $q$, hence the dot product is in the range $-dq^2$ to $\dimension q^2$. 
Thus, we require $p \geq 2\TotalUser \dimension^\approdegree q^{2\approdegree+1}+1$, where $\TotalUser$ is the total number of users, $\dimension$ is the dimension of the model updates, $q$ is the quantization parameter and $\approdegree$ is the degree of the discriminator function.

\section{\ourschemelo} \label{sec:LoByITFL}
Considering the high communication overhead required by ByITFL, we propose \ourschemelo, a Byzantine-resilient scheme with IT privacy and low communication cost. 
IT privacy is obtained by the use of threshold SecShare schemes. %
For clarity of exposition and without loss of generality, we consider an $(\TotalUser,\SSthreshold=\Colluding+1,\Colluding)$-SecShare scheme in the following, i.e., the vectors are not partitioned into sub-vectors and $\partition=1$. 

Before the training phase, we require an \textit{initialization phase} with a TTP. The TTP generates a sufficient number of Beaver triples $\gamma, \omega, \kappa$ and sends the corresponding shares $\gamma[i],\omega[i],\kappa[i]$ to the users, enabling multiplication of secret-shared messages in the training phase. In addition, it also samples sufficiently many vectors $\randomvector^1, \dots, \in \mathbb{F}_p^\dimension$ and values $\random^1, \dots, \in \mathbb{F}$ independently and uniformly at random. The $\randomvector$'s make the secret sharing step more efficient, as will be explained later. The random values $\random$'s are to compute $\aggresult$, as in \ourscheme. We omit the current global iteration index $\globalIteration$ for clarity. The numbers of random values and Beaver triples required for the training phase will be given later.
The TTP sends $\randomvector_i$, $\randomvector_j[i]$ for $j \in [\TotalUser]$, and $\random[i]$ to user $i \in [\TotalUser]$. 

\paragraph{Message Authentication Code} To prevent Byzantine users from providing corrupt computation results during the protocol, the TTP also assigns a one-time MAC~\cite{bendlin2011semi} for each secret share of the generated randomness. Suppose 
$s[i]$ is a secret share of a symbol $s \in \mathbb{F}_p$. A MAC takes a key pair $(\alpha, \beta)$ where $\alpha$ and $\beta$ are drawn independently and uniformly at random from $\mathbb{F}_p$, and is computed as $\operatorname{MAC}_{\alpha, \beta}(s[i])\define\alpha \cdot s[i]+\beta$. In our scheme, we keep $\alpha$ globally fixed and sample a new $\beta$ independently for each MAC. As a result, the MAC is additive homomorphic and linear operations $f$ can be performed on it, e.g., $f(s[i])$. Consider two parties $A$ (in our case, the user) and $B$ (in our case, the federator). 
Party $A$ holds the secret share $s[i]$ and the corresponding MAC $\operatorname{MAC}_{\alpha, \beta}(s[i])$. It computes the linear operation $f(s[i])$ and $f(\operatorname{MAC}_{\alpha, \beta}(s[i]))$, and sends the computations to party $B$. Party $B$ holds the MAC key $(\alpha, \beta)$ and wants to verify whether $f(s[i])$ is correctly computed. Upon receiving $f(s[i])$ and $f(\operatorname{MAC}_{\alpha, \beta}(s[i]))$, party $B$ computes $\operatorname{MAC}_{\alpha, \beta}(f(s[i]))$ and checks if it is consistent with $f(\operatorname{MAC}_{\alpha, \beta}(s[i]))$, which prevents the corrupted computations.
Thus, at the initialization phase, the TTP sends all the MAC keys $(\alpha, \beta)$'s to the federator, enabling integrity check of each computation on the secret shares. The federator marks users who do not pass the integrity check as Byzantine and exclude them from future computation. We omit $\alpha$ and $\beta$ later in the text for brevity. 

We denote $\{ s[i] \}$ as the secret shares of $s$ accompanied by the corresponding MAC, i.e. $\{ s[i] \}=\{ s[i], \operatorname{MAC}(s[i])\}$. Hence, after the initialization, each user $i \in [\TotalUser]$ receives a sufficient number of $\randomvector_{i}$, $\{ \randomvector_{j}[i] \}$ for $j \in [\TotalUser]$, $\{\random[i]\}$, and Beaver triples $\{\gamma[i]\}$, $\{\omega[i]\}$, $\{\kappa[i]\}$, while the federator receives all MAC keys $(\alpha, \beta)$'s used for generating the MACs of the random values. 

The \textit{training phase} of \ourschemelo also consists of the aforementioned main steps, as in \ourscheme.

\subsection{Normalization and Quantization} 
Users and the federator compute, normalize and quantize their model updates as described in \cref{subsec: normal}. The relation between $p$ and $q$ will be detailed later.

\begin{table*}[!t]
\centering
\caption{The number of random numbers and Beaver triples required per training iteration. 
We consider three different versions of Beaver triples: \textit{a)} $\gamma, \omega, \kappa$ are used for the multiplication between two scalars; \textit{b)} $\bm{\iota}, \bm{\phi}, \chi$ are used for computing the dot product of two $\dimension/\partition$-dimensional vectors, where $\bm{\iota}, \bm{\phi}\in \mathbb{F}_p^{\dimension/\partition}$ and $\chi$ is the dot product of $\bm{\iota}$ and $\bm{\phi}$; and \textit{c)} $\zeta, \bm{\xi}, \bm{\psi}$ are used for computing the element-wise multiplication of a scalar and a $\dimension/\partition$-dimensional vector, where $\bm{\xi}, \bm{\psi} \in \mathbb{F}_p^{\dimension/\partition}$ and $\bm{\psi}$ is the component-wise multiplication of $\zeta$ and $\bm{\xi}$. Each vector Beaver triple contains $\dimension/\partition$ scalar Beaver triples. The operations over vector Beaver triples can be easily generalized from scalar Beaver triples.}
\begin{tabular}{c|ccc}
                                & Notation  & Amount & Purpose (for $i \in [\TotalUser]$) \\ \hline 
\multirow{2}{*}{Random value}   & $\randomvector$  & $\TotalUser$   & secret share $\normalmodelupdate_i$ \\ \cline{2-4}
                                 & $\random$  & $1$   & hide $\sumone$ and $\sumtwo$ \\ \hline
\multirow{3}{*}{Beaver triple} & $\bm{\iota}, \bm{\phi}, \chi$ &  $\TotalUser$   & compute $\| \normalmodelupdate_i \|_2^2$\\ \cline{2-4}
                                & $\gamma, \omega,\kappa$ &  $(\approdegree-1)\TotalUser+1$  &  compute $\discripoly(\langle \normalmodelupdate_0,\normalmodelupdate_i \rangle)$ and $\random\sumone$\\ \cline{2-4}
                                & $\zeta, \bm{\xi}, \bm{\psi}$ &  $\TotalUser+1$  & compute $\discripoly(\langle \normalmodelupdate_0,\normalmodelupdate_i \rangle) \cdot \normalmodelupdate_i$ and $\random\sumtwo$ \\ \hline
\end{tabular}
\label{tbl:randomneeded}
\vspace{-0.1cm}
\end{table*}
    
\subsection{Sharing of the Normalized Model Updates}
The federator broadcasts $\normalmodelupdate_0$ to all users. 
Each user $i \in [\TotalUser]$ secret-shares $\normalmodelupdate_i$ to all other users with a threshold SecShare scheme, thus keeping it IT private while available for computations on the shares.
The secret sharing of $\normalmodelupdate_i$ is supported by the uniformly random vector $\randomvector_i$ distributed in the initialization phase. Specifically, user $i$ knows the plain value $\randomvector_i$ and each user $j \in [\TotalUser]$ has $\randomvector_i[j]$.
User $i$ broadcasts $\normalmodelupdate_i-\randomvector_i$ to all other users and user $j$ computes
\begin{equation}\label{eq:sharing}
    \{ \normalmodelupdate_i[j] \} = (\normalmodelupdate_i-\randomvector_i) + \{ \randomvector_i[j]\}.
\end{equation}
Note that this constitutes an $(\TotalUser, \Colluding+1, \Colluding)$-threshold SecShare of $\normalmodelupdate_i$ among users $j\in [\TotalUser]$ as detailed in the following.
The uniformly random vector $\randomvector_i$ is independent of $\normalmodelupdate_i$ and thus perfectly hides it by Shannon's one-time-pad~\cite{shannon1949communication}.
Once $\normalmodelupdate_i-\randomvector_i$ is public, gaining information about $\normalmodelupdate_i$ implies gaining information about $\randomvector_i$, which is only possible if the adversary has access to more than $\Colluding$ shares $\normalmodelupdate_i[j]$. Conversely, from any $\Colluding+1$ shares $\randomvector_i[j]$, it is straightforward to decode $\randomvector_i$ and thus $\normalmodelupdate_i$.
The corresponding $\operatorname{MAC}(\normalmodelupdate_i[j])$ is obtained by the additive homomorphism of the MAC. To summarize, each user $i$ obtains a secret share and the corresponding MAC of the local model update of every user $j\in [\TotalUser]$, i.e. $\{ \normalmodelupdate_j[i] \}$. 

\subsection{Validation of Normalization} 
To prevent Byzantine users from incorrectly normalizing their updates, validation is needed. Each user $i \in [\TotalUser]$ computes $ \{ \| \normalmodelupdate_j \|_2^2 [i]\} \in \mathbb{F}_p^\dimension$ by consuming $\dimension$ Beaver triples, and sends $\{ \| \normalmodelupdate_j \|_2^2[i]\}$ to the federator. The federator performs an integrity check on each $ \| \normalmodelupdate_j \|_2^2 [i]$, reconstructs $\| \normalmodelupdate_j\|_2^2 $ and checks if it lies within a certain interval. 
As the integrity check ensures the correctness of all computations, no error correction decoding~\cite{mceliece1981sharing} is needed, which is expensive in computation and requires a higher total number of users for decoding. Hence, we only require $\TotalUser \geq \Byzantine+\partition+\Colluding+\Dropout$.  Users who fail the normalization check are excluded from future computations. 
    
\subsection{Users Secure Computation}  
To compute the aggregation in  \cref{eq:aggre}, user $i$, having $\normalmodelupdate_0$, $\{ \normalmodelupdate_j[i] \}$ for $j \in [\TotalUser]$, and $\random[i]$, computes $\{ (\random \sumone)[i] \}$ and $\{ (\random \sumtwo)[i] \}$ by consuming Beaver triples and sends the computation results to the federator.

\subsection{Private Aggregation}
By receiving $\{ (\random \sumone)[i] \}$'s and $\{ (\random \sumtwo)[i] \}$'s from the users, the federator performs the integrity check to guarantee the correctness of the computation and reconstructs $\random \sumone$ and $\random \sumtwo$ by Lagrange interpolation, as long as it receives at least $\partition+\Colluding$ correct computations. Thus, we require $\TotalUser \geq \Byzantine+\partition+\Colluding+\Dropout$. After decoding, the federator computes $\aggresult= \frac{\random \sumtwo}{\random \sumone}$, converts $\aggresult$ to the real domain and follows \cref{eq:modelupdate} to compute the global model for the next iteration. Similarly, to guarantee the correctness of the reconstructed results, we require $p \geq 2\TotalUser \dimension^\approdegree q^{2\approdegree+1}+1$ to avoid any wrap-around in the finite field.

\begin{remark} 
The number of random values and Beaver triples required for the training phase (and distributed during the initialization phase) is given in \cref{tbl:randomneeded}.
\end{remark}

\begin{figure}[!t]
    \centering
    \vspace{-0.1cm}
    \resizebox{5.5cm}{!}{
    \input{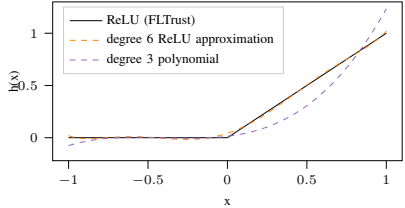}}
    \vspace{-.2cm}
    \caption{Comparison of different discriminator functions $\discripoly(x)$.}
    \vspace{-0.3cm}
    \label{fig:discriminator}
\end{figure}

\section{Experiments} \label{sec:experiments}

\subsection{Choice of the Discriminator Function}\label{sec:discriminator}
In FLTrust~\cite{cao2020fltrust}, the ReLU function is chosen as the discriminator function, whose use in IT privacy poses challenges. Specifically, the ReLU function requires comparison, which is difficult to compute under IT privacy guarantees.
The value of the comparison must remain private, since it inherently leaks information about the relative direction of local updates with the federator update. 
For this reason, we replace the discriminator function by one that can be computed in a private manner. A candidate is a polynomial approximation of the ReLU~\cite{xia2024byzantine}. However, while ReLU is plausible as a good discriminator function, there is no proof or principled reason to believe that the ReLU is the best discriminator function. 

\begin{figure}[!t]
  \centering
    \centering
    \resizebox{.75\linewidth}{!}{
    \vspace{-0.5cm}
    \input{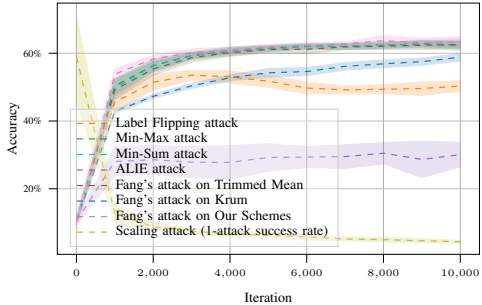}}
    \vspace{-0.2cm}
  \caption{Test accuracy (\%) of our schemes on non-i.i.d CIFAR-10 over iterations. For scaling attack, we plot ($1-$attack success rate). Baseline is $63.9\pm0.6\%$.}
  \label{fig:cifar10}
  \vspace{-0.5cm}
\end{figure}

\begin{table*}[!t]
    \centering
    \caption{Test accuracy (\%) of FLTrust~\cite{cao2020fltrust} and our proposed schemes (\ourscheme and \ourschemelo) on MNIST, Fashion MNIST and CIFAR-10 in i.i.d. and non-i.i.d. setting. For scaling attack, we give both test accuracy and attack success rate (\%). In the i.i.d. setting, baseline for MNIST is $96.2 \pm 0.1\%$, for Fashion MNIST is $88.6\pm0.1\%$, and for CIFAR-10 is $65.9\pm0.5\%$. In the non-i.i.d. setting, baseline is $96.2 \pm 0.1\%$, $88.6\pm0.2\%$, and $63.9\pm0.6\%$, respectively.}
    \label{exp:iid_noniid_combined}
    \vspace{-0.2cm}
    \renewcommand{\arraystretch}{1.2}
    \begin{tabular}{@{\hspace{0.01cm}}c@{\hspace{0.01cm}}|c|ccc|ccc}
    \multirow{2}{*}{Attack Type} & \multirow{2}{*}{Algorithm} & \multicolumn{3}{c}{Dataset (i.i.d.)} & \multicolumn{3}{|c}{Dataset (non-i.i.d.)}\\ 
    & & MNIST & Fashion MNIST & CIFAR-10 & MNIST & Fashion MNIST & CIFAR-10 \\ 
    \hline
    \multirow{2}{*}{\makecell{Label Flipping\\ Attack}} & FLTrust & $93.3\pm0.8$ & $87.6\pm0.3$ & $\bm{60.4\pm0.9}$ & $92.8\pm0.5$ & $86.8\pm0.4$ & \bm{$56.5\pm1.6$} \\ 
    & proposed &  $\bm{94.7\pm1.6}$ & $\bm{88.5\pm0.3}$ & $53.1\pm1.5$ & $\bm{94.1\pm1.3}$ & $\bm{88.1\pm0.4}$ & $50.3\pm1.5$ \\ 
    \hline
    \multirow{2}{*}{Min-Max Attack} & FLTrust &  $87.6\pm8.5$ & $86.5\pm1.2$ & $56.3\pm13.9$ & $87.5\pm5.9$ & $83.5\pm3.1$ & $59.9\pm0.8$ \\ 
    & proposed & $\bm{93.1\pm4.5}$ & $\bm{88.7\pm0.3}$ & $\bm{65.6\pm0.7}$ & $\bm{95.1\pm1.6}$ & $\bm{88.3\pm0.3}$ & $\bm{62.4\pm1.6}$ \\ 
    \hline
    \multirow{2}{*}{Min-Sum Attack} & FLTrust & $88.2\pm7.4$ & $84.6\pm3.9$ & $60.3\pm8.8$ & $85.2\pm7.3$ & $84.8\pm3.1$ & $59.4\pm0.9$ \\ 
    & proposed & $\bm{95.1\pm1.5}$ & $\bm{88.9\pm0.2}$ & $\bm{65.0\pm1.7}$ & $\bm{95.6\pm0.8}$ & $\bm{88.5\pm0.4}$ & $\bm{62.5\pm1.1}$ \\ 
    \hline
    \multirow{2}{*}{ALIE Attack} & FLTrust & $\bm{92.7\pm0.8}$ & $\bm{66.4\pm23.2}$ & $20.9\pm12.0$ & $92.6\pm0.8$ & $\bm{42.4\pm16.0}$ & $25.0\pm8.4$ \\ 
    & proposed & $92.6\pm1.8$ & $32.7\pm3.8$ & \bm{$27.1\pm12.2$} & $\bm{92.8\pm1.7}$ & $28.9\pm7.3$ & $\bm{30.1\pm3.6}$ \\ 
    \hline
    \multirow{2}{*}{\makecell{Fang’s Attack \\ on Trimmed Mean}} & FLTrust & $89.1\pm4.1$ & $84.2\pm2.0$ & $62.9\pm0.7$ & $91.6\pm1.3$ & $85.3\pm1.5$ & $59.8\pm1.0$ \\ 
    & proposed & $\bm{95.1\pm1.5}$ & $\bm{86.7\pm3.6}$ & $\bm{64.9\pm0.8}$ & $\bm{92.9\pm4.1}$ & $\bm{86.9\pm1.8}$ & $\bm{62.4\pm0.9}$ \\ 
    \hline
    \multirow{2}{*}{\makecell{Fang’s Attack\\ on Krum}} & FLTrust & $\bm{92.9\pm0.5}$ & $\bm{86.8\pm0.6}$ & $\bm{62.9\pm0.8}$ & $\bm{93.4\pm0.4}$ & $86.4\pm0.5$ & \bm{$59.1\pm1.2$} \\ 
    & proposed & $91.3\pm2.5$ & $86.7\pm1.0$ & $62.7\pm0.9^{\star}$ & $91.6\pm1.8$ & $\bm{86.8\pm0.6}$ & $58.8\pm1.2^{\star}$ \\ 
    \hline
    \multirow{2}{*}{\makecell{Fang’s Attack on\\ FLTrust/Our Schemes}} & FLTrust & $90.8\pm1.2$ & $87.5\pm0.3$ & $62.2\pm1.4$ & $91.9\pm2.0$ & $87.0\pm0.3$ & $60.2\pm1.3$ \\ 
    & proposed & $\bm{92.9\pm1.2}$ & $\bm{88.5\pm0.2}$ &  $\bm{65.6\pm1.8}$ & $\bm{93.8\pm1.4}$ & $\bm{88.6\pm0.3}$ & $\bm{63.4\pm1.5}$ \\ 
    \hline
    \multirow{2}{*}{\makecell{Scaling Attack\\ (Test Accuracy)}} & FLTrust & $93.0\pm1.1$ & $87.6\pm0.3$  & $65.2\pm0.9$ & $93.3\pm0.9$ & $87.2\pm0.6$  & $62.0\pm1.2$ \\ 
    & proposed &  $93.7\pm1.6$ & $89.0\pm0.3$  & $68.4\pm0.9$  &  $93.5\pm1.3$  &  $88.9\pm0.3$  &  $66.4\pm1.0$ \\ \hdashline
    \multirow{2}{*}{\makecell{Scaling Attack\\ (Attack Success Rate)}} & FLTrust & $\bm{61.7\pm36.6}$ &  $99.9\pm0.1$ & \bm{$92.0\pm1.1$} & $\bm{67.1\pm32.1}$ & $\bm{99.8\pm0.1}$ & \bm{$91.9\pm0.8$} \\ 
    & proposed & $66.8\pm40.7$ & $99.9\pm0.1$ & $95.2\pm0.6$ & $86.3\pm27.0$ & $99.9\pm0.1$ & $95.6\pm0.4$ \\ 
\end{tabular}
    \captionsetup{justification=justified,parskip=1pt}
    \caption*{\footnotesize * Results marked with a star are from a slightly modified version of our proposed scheme, i.e., not using \cref{eq:modelupdate}, but using $\aggresult$ directly as the model update. Using \cref{eq:modelupdate} causes the test accuracy to drop to $17.8\pm1.9\%$ and $17.9\pm2.9\%$ for i.i.d. and non-i.i.d. settings, respectively.}
    \vspace{-0.3cm}
\end{table*}

Therefore, we design our discriminator function to be a polynomial that satisfies the following core observation: 
The values of the discriminator function (TSs) for negative inputs need not be \emph{exactly} zero; it is sufficient that updates in the opposite direction of the federator's update have small values.
Further, even moderate negative values in the discriminator function for $x<0$ cannot be exploited by the adversary. Consider some $x_0>0$ such that $h(-x_0) < 0 < h(x_0)$ and $|h(-x_0)|<|h(x_0)|$. 
Then, rather than sending a model update with cosine similarity $-x_0$ the adversary can always achieve a higher TS, and thus influence the aggregate more, by flipping the direction of the model update, achieving cosine similarity $x_0$.
The direction of the model update weighted by the trust score is the same in either case.
Hence, such corrupt updates cannot harm the learning process beyond what corrupt model updates with a positive cosine similarity are capable of. 

Based on these observations, we find that it is not required to accurately approximate the ReLU function by high-degree polynomials as in~\cite{xia2024byzantine}, which is only satisfactory when the degree $\approdegree$ is greater than $6$; instead, we carefully choose a degree-$3$ polynomial that mimics ReLU in the negative half. In the positive half, we choose a more conservative shape than ReLU by giving higher trust scores to local model updates that strongly point in the same direction as the federator model update. Uncertain local model updates are attenuated more aggressively. This is very similar to~\cite{lu2023robust}, where the authors use a quadratic function on the right, and a constant $0$ on the left. As a by-product, the degree of our discriminator polynomial is decreased significantly to $3$. The chosen degree-$3$ polynomial is given by $\discripoly(x) = 0.46897526 \bm{x}^3 + 0.56578977 \bm{x}^2 + 0.1860353 \bm{x} + 0.01363545$, and plotted in \cref{fig:discriminator} alongside with the degree-$6$ polynomial of \cite{xia2024byzantine} and the ReLU function.

\begin{figure}[tb]
  \centering
    \includegraphics[width=0.43\textwidth]{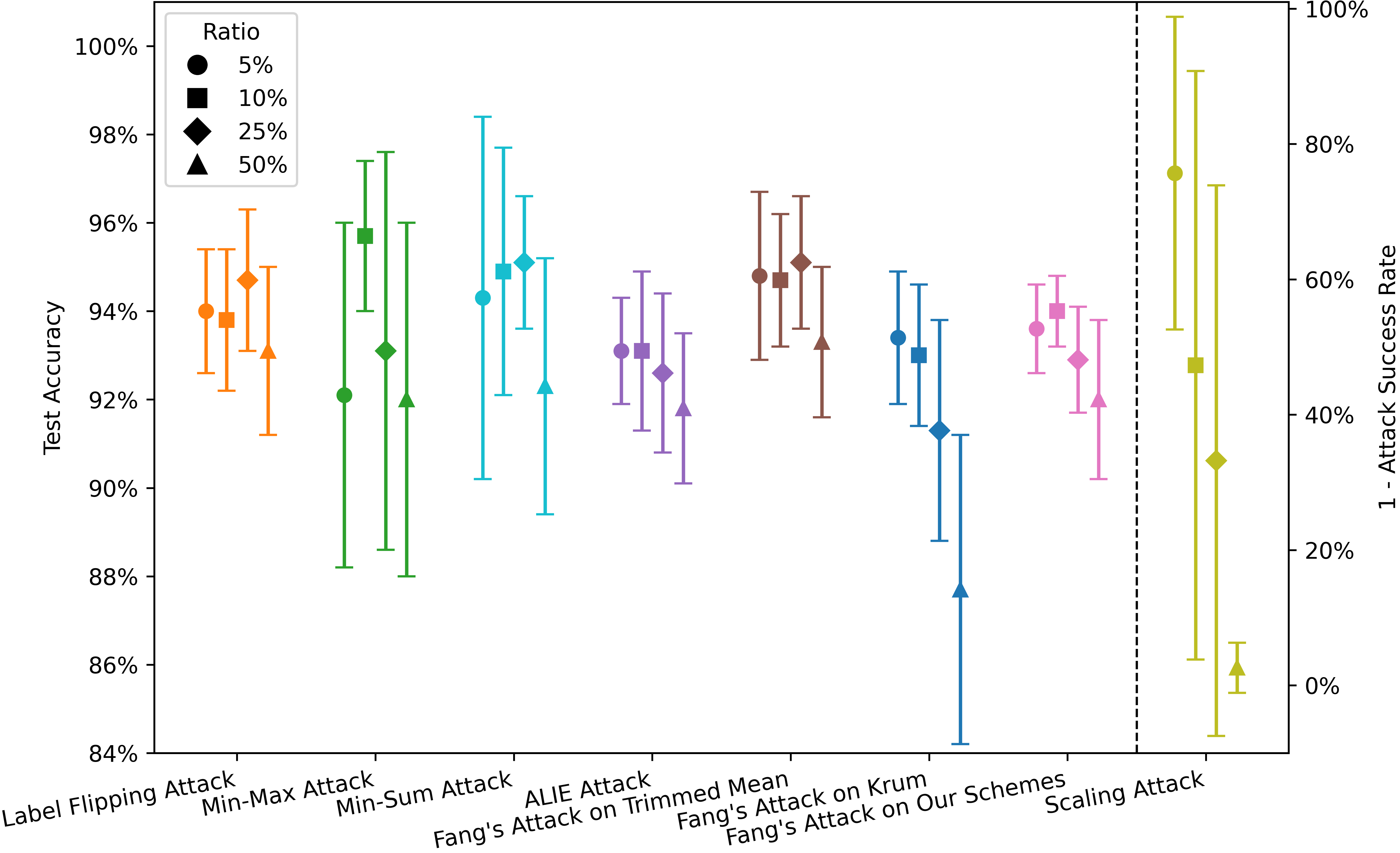}
    \caption{Test accuracy of our schemes on MNIST with $40$ users and i.i.d. setting. For scaling attack, we plot $1-$attack success rate and separate it with a dashed line. Baseline is $96.2\pm0.1\%$. }
    \vspace{-0.6cm}
  \label{fig:ratio}
\end{figure}

\subsection{Numerical Results}
We demonstrate the convergence of our algorithms, using the above polynomial as a discriminator, and compare them to FedAvg and FLTrust. Our experiments are based on the implementation provided by~\cite{gehlhar2023SAFEFL}.
We consider MNIST~\cite{deng2012mnist}, Fashion MNIST~\cite{xiao2017fashion} and CIFAR-10~\cite{krizhevsky2009learning} distributed across $\TotalUser=40$ users. On MNIST and Fashion MNIST, we train a three-layer dense neural network, and on CIFAR-10 a CNN model. The rectified linear unit function (ReLU) is used as the activation function. As loss function we choose Cross-Entropy. In each iteration, users randomly sample a minibatch of 64 samples from their local training dataset and perform local training on the minibatch. The learning rate for MNIST and Fashion MNIST is chosen to be $0.1$, and $0.01$ for CIFAR-10. 

We assume $25\%$ of the users are Byzantine ($\Byzantine=10$) and perform the Label Flipping attack, Min-Max attack~\cite{shejwalkar2021manipulating}, Min-Sum attack~\cite{shejwalkar2021manipulating}, a-little-is-enough (ALIE) attack~\cite{baruch2019little}, Fang's attack~\cite{fang2020local} on Trimmed Mean~\cite{yin2018byzantine}, Krum~\cite{blanchard2017machine} and different discriminator functions, and Scaling attack. The label flipping attack follows the same setting as in \cite{fang2020local}. Fang's attack~\cite{fang2020local} is an untargeted local model poisoning attack, optimized for different robust aggregation rules, such as Krum and Trimmed Mean. The scaling attack is equivalently known as the backdoor attack~\cite{bagdasaryan2020backdoor}. For scaling attack, we use attack success rate, i.e., the fraction of the attacker's target testing samples being predicted as the attacker's chosen label, to measure the performance of the attack, rather than using test accuracy. A lower attack success rate means a more robust defense.
As in \cite{cao2020fltrust}, we randomly split the users into $10$ groups and a training example with label $j$ is assigned to group $j$ with probability $a>0$ and to any other groups with probability $\frac{1-a}{9}$. Data are uniformly distributed to each user within the same group. Therefore, we set $a=0.1$ for i.i.d. setting and set $a=0.5$ for non-i.i.d. setting. 
We set $q=1024$, $|\Dataset_0|=100$ as in~\cite{cao2020fltrust}, and $\varepsilon=0.02$ for the normalization validation. 

The performance of our schemes using the aforementioned degree $3$ polynomial on CIFAR-10 under different attacks is shown in \cref{fig:cifar10}, and the experiment results presented in \cref{exp:iid_noniid_combined} show that our schemes are comparable to FLTrust. This demonstrates that our approach maintains the original performance of the Byzantine resilience while incorporating IT privacy. Note that, for CIFAR-10, we found that a slight adjustment to our schemes significantly improved the resilience against Fang's attack on Krum. We skip the normalization and flipping step in \cref{eq:modelupdate}, and directly use $\aggresult$ as the model update, which is exactly the algorithm proposed in our previous work~\cite{xia2024lobyitfl}. This variation is minor and does not affect generalizability, and we report the results of this variation in \cref{fig:cifar10} and \cref{exp:iid_noniid_combined}. 
\cref{fig:ratio} shows the performance of our schemes under attacks with different ratios of Byzantine users.

\section{Conclusion}
We proposed two robust aggregation schemes for FL that guarantee IT privacy. We require the federator to hold a root dataset for reference, used to scale the users' local updates during aggregation. To achieve IT privacy, we used a suitable polynomial as the discriminator function to compute trust scores and leverage secret sharing techniques to ensure privacy. We provided theoretical guarantees on privacy, resilience, convergence, and algorithmic complexity, along with empirical validation demonstrating convergence even in adversarial settings. 
While our primary contribution is to bring IT privacy with Byzantine resilience, we note that Byzantine resilience approaches based on cosine-similarity are still not fully understood, and a systematic and careful study remains an open problem. Consequently, the cosine-similarity may have additional weaknesses beyond the known weakness to backdoor attacks~\cite{lu2023robust}. Establishing meaningful theoretical guarantees for robust aggregation with less restrictive assumptions, is yet unsolved. 
Future work includes extending our approach to wireless FL environments, as explored in \cite{egger2023private,zhang2025fundamental}. 

\appendix

\subsection{Proof for \cref{thm:theorem_by}} \label{app:ByITFL}
\begin{proof}
We start by proving the IT privacy of \ourscheme.
\subsubsection{IT Privacy}\label{proof:privacy_by}
We first prove the IT privacy against any $\Colluding$ colluding users and the federator according to  \cref{eq:privacy_combine}:
\begin{align}\label{eq:proof_privacy_byitfl}
    & I(\modelupdate_{\honestset}; M_{\mathcal{T}}, M_0\!\mid \!\modelupdate_{\mathcal{T}}, \modelupdate_0, \Dataset_{\mathcal{T}}, \Dataset_0, \model, \aggresult) \nonumber \\
    = & H(\modelupdate_{\honestset}\!\mid \!\modelupdate_{\mathcal{T}}, \modelupdate_0, \Dataset_{\mathcal{T}}, \Dataset_0, \model, \aggresult) \nonumber \\
    & - H(\modelupdate_{\honestset}\!\mid \! M_{\mathcal{T}}, M_0,\modelupdate_{\mathcal{T}}, \modelupdate_0, \Dataset_{\mathcal{T}}, \Dataset_0, \model, \aggresult) \nonumber \\
    = & H(\modelupdate_{\honestset}\!\!\mid\! \!\Dataset_{\mathcal{T}},\! \Dataset_0,\! \model,\! \aggresult) \!- \!H(\modelupdate_{\honestset}\!\!\mid \!\! M_{\mathcal{T}},\! M_0,\! \Dataset_{\mathcal{T}},\! \Dataset_0,\! \model,\! \aggresult),
\end{align}
where the last equation follows because $\modelupdate_\colset$ is a deterministic function of $\Dataset_{\mathcal{T}}$ and $\model$, and $\modelupdate_0$ is a deterministic function of $\Dataset_0$ and $\model$. We then consider the exchanged messages $M_\colset$ observed by the colluding users and $M_0$ observed by the federator in each step. $M_\colset$ includes:
\begin{enumerate}[label=Step \Alph*, wide=0pt, font=\itshape,, leftmargin=*]
    \item [Step A and E:] nothing, since step A only involves local computations and step E is performed by the federator;
    \item [Step B:] shares $\normalmodelupdate_j[i]$ for $i \in \colset, j \in [\TotalUser]$; 
    \item [Step C:] ${\| \normalmodelupdate_j \|}_2^2[i]$ for $i \in \colset, j \in [\TotalUser]$ and sub-shares from the re-randomization when computing ${\| \normalmodelupdate_j \|}_2^2[i]$; and
    \item [Step D:] shares $\sumone[i], \sumtwo[i]$, $(\random \sumone)[i]$, $(\random \sumtwo)[i]$, the random values $\random_i$, and $\random_j[i]$, for $j\!\in\! [\TotalUser]$ and $i \!\in\! \colset$, and sub-shares from the re-randomization step when computing shares of $\sumone, \sumtwo, \random \sumone$ and $\random \sumtwo$.
\end{enumerate}
With regard to $M_0$, we need to consider: 
\begin{enumerate}[label=Step \Alph*, wide=0pt, font=\itshape,, leftmargin=*]
    \item [Step A and B:] nothing, since step A and B only involve computations among the users; and
    \item [Step C, D and E:] shares ${\| \normalmodelupdate_j \|}_2^2[i]$, for $j \in [\TotalUser]$, $(\random \sumone)[i]$ and $(\random \sumtwo)[i]$ received from the users for $i \in [\TotalUser]$, and the reconstruction of ${\| \normalmodelupdate_j \|}_2^2$, $\random \sumone$ and $\random \sumtwo$.
\end{enumerate}

Regarding $M_\colset$, we leverage the privacy guarantees of LCC \cite{yu2019lagrange}, the re-randomization step \cite{gennaro1998simplified}, \cite{asharov2017full} and ITVSS \cite{BGW}. Since each secret-sharing is encoded by a degree $\partition+\Colluding-1$ polynomial, when considering at most $\Colluding$ colluding users, the shares and the sub-shares observed by the colluding users are completely random and independent of $\modelupdate_\honestset, \Dataset_\colset, \Dataset_0, \model$ and $\aggresult$. Since the random values are drawn uniformly and independently from $\mathbb{F}_p$, $\random_i$'s are independent of $\modelupdate_\honestset, M_0, \Dataset_\colset, \Dataset_0, \model$ and $\aggresult$. However, the shares ${\| \normalmodelupdate_j \|}_2^2[i]$, $(\random \sumone)[i]$ and $(\random \sumtwo)[i]$ for $i \in \colset, j \in [\TotalUser]$ are not independent of $M_0$, rather, $M_0$ includes these shares. We denote these shares as $\hat{M}_{\colset}$, thus, we obtain $H(\hat{M}_{\colset}\!\mid\!M_0)=0$. Hence,
\begin{align}
\label{eq:mt}
    H(\modelupdate_{\honestset} \!\! \mid \!\! M_{\mathcal{T}},\! M_0,\! \Dataset_{\mathcal{T}},\! \Dataset_0,\! \model,\! \aggresult) \nonumber & \!\!=\!\! H(\modelupdate_{\honestset}\!\!\mid\!\! \hat{M}_{\colset},\! M_0,\! \Dataset_{\mathcal{T}},\! \Dataset_0,\! \model,\! \aggresult) \nonumber \\
    =& H(\modelupdate_{\honestset} \!\!\mid\!\!  M_0,\! \Dataset_{\mathcal{T}},\! \Dataset_0,\! \model,\! \aggresult),
\end{align}
where the first equality holds because of the independence, and the second equality holds because $H(\hat{M}_{\colset} \mid M_0)=0$. 

As for $M_0$, since the collection of all the shares from user $i \in [\TotalUser]$ is informationally equivalent to the secret itself, we only need to consider the privacy of the reconstructions. Among the reconstructions: \begin{enumerate*}[label=(\emph{\alph*})]
    \item $\| \normalmodelupdate_i \|_2^2$ lies within a certain range for all possible model updates (ideally equivalent to one);
    \item $\sumone$ is completely hidden due to the randomness $\random$; and
    \item no information is leaked beyond $\aggresult=\frac{\sumtwo}{\sumone}$, i.e., $H(\random\sumtwo|\aggresult)=0$.
\end{enumerate*}
We denote all the exchanged messages except $\random\sumtwo$ as $\hat{M}_0$, thus,
\begin{align*}
    H(\modelupdate_{\honestset}\!\!\mid \!\! M_0, \Dataset_{\mathcal{T}}, \Dataset_0, \model, \aggresult)& =H(\modelupdate_{\honestset}\!\!\mid \!\! \hat{M}_0, \random\sumtwo, \Dataset_{\mathcal{T}}, \Dataset_0, \model, \aggresult) \\
    &= H(\modelupdate_{\honestset}\!\!\mid \!\! \Dataset_{\mathcal{T}}, \Dataset_0, \model, \aggresult),
\end{align*}
where the second equation holds since $\hat{M}_0$ is completely random and independent of other terms, and $H(\random\sumtwo|\aggresult)=0$.
By substituting the above into  \cref{eq:proof_privacy_byitfl} we prove  \cref{eq:privacy_combine}, showing that \ourscheme is IT private against $\Colluding$ colluding users and the federator.

\subsubsection{Resilience against Byzantine and Robustness against dropout}\label{proof:byz_by}
Byzantine users may present corrupt updates, and perform the required computations dishonestly. The effect of corrupt updates is mitigated by the robust aggregation rule, i.e. FLTrust~\cite{cao2020fltrust}. Next, we guarantee the correct computation of $\random\sumone[i]$ and $\random\sumtwo[i]$. For dishonest computations, Byzantine users may behave arbitrarily in any step during the computation. Particularly, user $i\in\byzset$ can: 
\begin{enumerate*}[label=(\emph{\alph*})]
    \item incorrectly normalize local model updates $\realnormalmodelupdate_i$ in the normalization step;
    \item distribute invalid secret shares, i.e. secret shares encoded using different encoding polynomials, in the secret sharing step; and,
    \item misbehave when computing $ \| \normalmodelupdate_j \|_2^2 [i]$, $ (\random\sumone)[i] $ and $(\random\sumtwo)[i] $, and send corrupt results to the federator. 
\end{enumerate*} 

Byzantine behavior in the first scenario can be detected and identified in the normalization validation step. For all honest users, the model updates presented and normalized lie within a certain range. Thus, even if malicious users present model updates with the largest squared $l_2$-norm that will be accepted by the normalization validation, it only has limited impact since $\epsilon$ is empirically set to be very small. For the second scenario, the attack can be detected by ITVSS~\cite{BGW}, thus guaranteeing the correctness of the secret sharing scheme. 

With regard to the third scenario, the correctness is ensured by the Reed-Solomon (RS) decoding algorithm~\cite{mceliece1981sharing}. As described in \cref{sub:normalization} and \cref{sub:aggregation}, $ \| \normalmodelupdate_j \|_2^2 [i]$, $ (\random\sumone)[i] $ 
and $ (\random\sumtwo)[i] $ can be viewed as evaluations of the polynomials with degree $2(\partition+\Colluding-1)$, $(\approdegree+1) \cdot (\partition+\Colluding-1)$ and $(\approdegree+2)\cdot (\partition+\Colluding-1)$, respectively. The decoding of $ \| \normalmodelupdate_j \|_2^2$, $\random\sumone$ 
and $\random\sumtwo$ can be treated as the decoding of Reed-Solomon codes with parameters 
\begin{enumerate}
    \item $[\TotalUser, 2(\partition+\Colluding-1)+1, \TotalUser-2(\partition+\Colluding-1)]_p$;
    \item $[\TotalUser, (\approdegree+1)\! \cdot\! (\partition+\Colluding-1)+1, \TotalUser-(\approdegree+1) \!\cdot\! (\partition+\Colluding-1)]_p$;
    \item $[\TotalUser, (\approdegree+2)\!\cdot\! (\partition+\Colluding-1)+1, \TotalUser-(\approdegree+2)\!\cdot\! (\partition+\Colluding-1)]_p$;
\end{enumerate}
respectively. Note that, a $[\TotalUser,\SSthreshold,\TotalUser-\SSthreshold+1]_p$ RS code with $\Dropout$ erasures is able to decode $\lfloor \frac{\TotalUser-\SSthreshold-\Dropout}{2}\rfloor$ errors. By viewing the updates presented by Byzantine users as errors in a RS code and dropout users as erasures, we obtain three requirements on $\TotalUser$ given by the decoding of these three RS codes:
\begin{enumerate}
    \item $\Byzantine \leq \frac{\TotalUser-\left(2(\partition+\Colluding-1)+1\right)-\Dropout}{2}$;
    \item $\Byzantine \leq \frac{\TotalUser-\left((\approdegree+1) \cdot (\partition+\Colluding-1)+1\right)-\Dropout}{2}$;
    \item $\Byzantine \leq \frac{\TotalUser-\left((\approdegree+2)\cdot (\partition+\Colluding-1)+1\right)-\Dropout}{2}$;
\end{enumerate}
Hence, the federator can decode the correct computation results as long as $\TotalUser \geq 2\Byzantine+(\approdegree+2)\cdot (\partition+\Colluding-1)+\Dropout+1$, which is the largest $\TotalUser$ required by these three decodings.

\subsubsection{Complexity}
For each user sharing a single scalar, the computation complexity of encoding in LCC is $O(n\log^2 n\log\log n)$ \cite{yu2019lagrange} and that of ITVSS \cite{BGW} is $O(n^2 \log^2 n)$. Since the re-randomization step \cite{gennaro1998simplified} involves sub-sharing each secret share using ITVSS and a linear combination of the sub-shares, the computation complexity is $O(n^2\log^2 n\log\log n+n^3 \log^2 n)$. The communication cost for the ITVSS is $O(n^2)$, and $O(n^3)$ for re-randomization. With respect to the federator, the computation cost for decoding a single scalar using error-correction decoding takes $O(n\log^2 n\log\log n)$. The remaining proof follows from counting, here omitted for brevity.
\end{proof}

\subsection{Proof for \cref{thm:theorem_loby}} \label{app:LoByITFL}
\begin{proof}
We begin with the proof for IT privacy for \ourschemelo.
\subsubsection{IT Privacy}
For privacy against any $\Colluding$ colluding users and the federator according to \cref{eq:privacy_combine}, we want to prove \cref{eq:proof_privacy_byitfl}. The exchanged messages $M_{\mathcal{T}}$ are:
\begin{enumerate}[label=Step \Alph*, wide=0pt, font=\itshape,, leftmargin=*]
    \item [Initialization Step:] random vectors $ \randomvector_i$'s, shares of random values $\{ \randomvector_j[i] \}$'s for $j \in [\TotalUser]$ and $\{ \random[i] \}$'s, shares of Beaver triple $\{ \gamma[i] \}$'s, $\{ \omega[i] \}$'s, $\{ \kappa[i] \}$'s, for $i\in \colset$;
    \item [Step A and E:] nothing, since step A only involves local computations and step E is performed by the federator;
    \item [Step B:] the broadcasted values $\normalmodelupdate_j-\randomvector_j$ and the shares $\{ \normalmodelupdate_j[i] \}$, for $i \in \colset, j \in [\TotalUser]$; 
    \item [Step C:] shares $\{\| \normalmodelupdate_j \|_2^2[i]\}$ and the messages incurred when consuming Beaver triples to compute $\{\| \normalmodelupdate_j \|_2^2[i]\}$, for $i \in \colset, j \in [\TotalUser]$; and
    \item [Step D:] shares $\{\sumone[i]\}, \{\sumtwo[i]\}$, $\{(\random \sumone)[i]\}$, $\{(\random\sumtwo)[i]\}$, and $\{\random_j[i]\}$ for $j\in [\TotalUser]$ and $i \in \colset$, and the messages incurred when consuming Beaver triples to compute shares of $\sumone, \sumtwo, \random \sumone$ and $\random\sumtwo$.
\end{enumerate}
The exchanged messages $M_0$ include:
\begin{enumerate}[label=Step \Alph*, wide=0pt, font=\itshape,, leftmargin=*]
    \item [Initialization Step:] the MAC keys used for generating the MACs of the secret shares;
    \item [Step A and B:] nothing, since step A and B only involve computations among the users; and
    \item [Step C, D and E:] shares $\{ {\| \normalmodelupdate_j \|}_2^2[i] \}$ for $j \in [\TotalUser]$, $\{ (\random \sumone)[i] \}$ and $\{ (\random\sumtwo)[i] \}$ received from the user $i \in [\TotalUser]$, the reconstruction of ${\| \normalmodelupdate_j \|}_2^2$, $\random \sumone$ and $\random\sumtwo$, and the messages incurred when using Beaver triples to compute multiplications.
\end{enumerate}
Note that, when computing $\{ (sv)[i] \}$, colluding users get $\{s[i]-\gamma[i]\}$, $\{v[i]-\omega[i]\}$, $s-\gamma$ and $v-\omega$, for $i\in \colset$, and the federator gets the values of $\{s[i]-\gamma[i]\}$, $\{v[i]-\omega[i]\}$ for $i \in [\TotalUser]$ and the reconstructed values $s-\gamma$ and $v-\omega$.

We leverage the privacy guarantees of the SecShare scheme, the Beaver triple and the MAC scheme. Regarding the messages $M_\colset$, when the number of colluding users is smaller than the threshold of the secret sharing scheme, i.e. $\Colluding$, 
we have:
\begin{enumerate*}[label=(\emph{\alph*})]
    \item the shares observed by the colluding users are completely random and independent of $\modelupdate_\honestset, \Dataset_\colset, \Dataset_0, \model$ and $\aggresult$, but partially included by $M_0$;
    \item the random vectors $\randomvector$'s are random and independent of $M_0, \modelupdate_\honestset, \Dataset_\colset, \Dataset_0, \model$ and $\aggresult$;%
    \item the broadcasted values $\normalmodelupdate_j-\randomvector_j$'s are random because of one-time padding;
    \item the MACs are also random because of the randomness of $\beta$; and
    \item the messages incurred when consuming Beaver triples for multiplication are either secret shares or protected by one-time padding, thus, independent of $\modelupdate_\honestset, M_0, \Dataset_\colset, \Dataset_0, \model$ and $\aggresult$.
\end{enumerate*}
Thus, we have \cref{eq:mt} when considering the shares $\{ {\| \normalmodelupdate_j \|}_2^2[i] \}$, $\{ (\random \sumone)[i] \}$ and $\{ (\random\sumtwo)[i] \}$ for $j \in [\TotalUser]$ and $i \in \colset$ as $\hat{M}_\colset$.

As for $M_0$, the MAC keys sent from the TTP and the messages broadcasted during multiplication are random values that are independent of $\modelupdate_\honestset, \Dataset_\colset, \Dataset_0, \model$ and $\aggresult$. By following the same reasoning as in \cref{proof:privacy_by}, we conclude the IT privacy proof for \ourschemelo. 

\subsubsection{Resilience against Byzantine and Robustness against dropout}
Byzantine users may present corrupt updates, and perform the required computations dishonestly. Same as in \cref{proof:byz_by}, the effect of corrupt updates is eliminated by the robust aggregation rule. We next guarantee the correct computation of $\random\sumone[i]$ and $\random\sumtwo[i]$. For dishonest computations, Byzantine users $i\in\byzset$ may:
\begin{enumerate*}[label=(\emph{\alph*})]
    \item incorrectly normalize local model updates $\realnormalmodelupdate_i$; 
    \item compute invalid secret shares $\{ \normalmodelupdate_j[i] \}$ for any $j\in[\TotalUser]$; and,
    \item misbehave during the computation of $\{ \| \normalmodelupdate_j \|_2^2 [i]\}$, $\{ (\random\sumone)[i] \}$ 
and $\{ (\random\sumtwo)[i] \}$.
\end{enumerate*}

As stated in \cref{proof:byz_by}, Byzantine attacks in the first scenario can be detected and identified in the normalization validation step. For Byzantine attacks in the remaining scenarios, the MACs are used to enable integrity check, i.e., for each linear operation performed on the secret shares, the users also perform the same computation on the corresponding MACs. 
Since the federator receives all the corresponding MAC keys $(\alpha,\beta)$'s in the initialization phase, integrity check can be performed by the federator to guarantee the correctness of the computation. 
Therefore, the correctness of the computations entirely relies on the integrity check of the IT one-time MAC~\cite{bendlin2011semi}, where we need to consider the forgery probability of the MAC. Forgery probability is the probability that a malicious user guesses the MAC key used for generating the MAC of a specific secret share correctly, such that it can create a valid MAC for the computation result to fool the integrity check. Since in all MAC keys $(\alpha,\beta)$s', $\alpha$ is a uniformly random consistent value and $\beta$'s are chosen independently uniformly at random from the underlying prime field $\mathbb{F}_p$, which are used to protect the value of $\alpha$ when we reusing it in different MACs. The forgery probability for each scalar is the probability of guessing the $\beta$ used for generating the MAC for a specific scalar, that is $1/ \| \mathbb{F}_p \|= 1/p$, i.e. the inverse of the size of the underlying finite field. Since the finite field we choose is very large, i.e. $p \geq 2\TotalUser \dimension^\approdegree q^{2\approdegree+1}+1$, and the malicious users have to guess the $\beta$'s used for all $\dimension$ dimensions to fool the integrity check, the forgery probability is small enough to guarantee the security and correctness of the scheme.

\subsubsection{Complexity}
In the training phase, the computation for each user sharing a scalar is, according to  \cref{eq:sharing}, the addition of a publicly known value and a secret share, thus $O(1)$, and that of one share multiplication is $O(1)$ as well by consuming one Beaver triple, which only involves addition and scaling of the secret shares. For the federator, the computation is to reconstruct the intermediate results during share multiplications and the final results, which is a Lagrange polynomial interpolation problem and costs $O(n\log^2n\log\log n)$ per scalar. The communication cost for secret sharing and computations on secret shares are $O(1)$ for each user and $O(n)$ for the federator per scalar. The rest of the proof follows by counting, and is omitted here for the sake of brevity.
\end{proof}

\subsection{Proof for convergence}\label{proof:convergence}
For the proof of convergence, we need \cref{lem:bounded}.
\begin{lem}\label{lem:bounded}
    At each global iteration $\globalIteration$, the distance between the federator's aggregation of the gradients $\modelupdate^{(\globalIteration)}$ and the gradient $\nabla F(\model^{(\globalIteration)})$ is bounded from above as:
    \[
    \| \modelupdate^{(\globalIteration)} - \nabla F(\model^{(\globalIteration)}) \| \leq
    3\| \modelupdate_0^{(\globalIteration)} - \nabla F(\model^{(\globalIteration)}) \| + 2 \|\nabla F(\model^{(\globalIteration)}) \|.
    \]
\end{lem}

\begin{proof}
We omit the superscript ${(\globalIteration)}$ for ease of notation. 
Recall the definition of  $\aggresult$
\[
\aggresult \define \frac{1}{\sum_{i\in [\TotalUser]}{\discripoly(\langle \normalmodelupdate_0,\normalmodelupdate_i \rangle)}}\sum_{i\in [\TotalUser]}{\discripoly(\langle \normalmodelupdate_0,\normalmodelupdate_i \rangle) \cdot \normalmodelupdate_i}.
\]
Therefore, since \cref{eq:modelupdate},  $\langle \modelupdate,\modelupdate_0 \rangle \geq 0$ always holds.
We can bound $\| \modelupdate - \nabla F(\model) \|$ as
\begin{align*}
\| \modelupdate - \nabla F(\model) \|
    &= \| \modelupdate - \modelupdate_0 + \modelupdate_0 - \nabla F(\model) \| \\
    &\stackrel{(a)}{\leq} \| \modelupdate - \modelupdate_0 \| + \| \modelupdate_0 - \nabla F(\model) \| \\
    &\stackrel{(b)}{\leq} \| \modelupdate + \modelupdate_0 \| + \| \modelupdate_0 - \nabla F(\model) \| \\
    &\stackrel{(c)}{\leq} \| \modelupdate \| + \| \modelupdate_0 \| + \| \modelupdate_0 - \nabla F(\model) \| \\
    &= 2\| \modelupdate_0 \| + \| \modelupdate_0 - \nabla F(\model) \| \\
    &= 2\| \modelupdate_0 - \nabla F(\model) + \nabla F(\model) \| \\
    & ~~~~ + \| \modelupdate_0 - \nabla F(\model) \| \\
    &\stackrel{(d)}{\leq} 3\| \modelupdate_0 - \nabla F(\model) \| + 2 \|\nabla F(\model) \|,
\end{align*}
where (a), (b), and (d) follow from the triangle inequality, and (c) follows since $\langle \modelupdate,\modelupdate_0 \rangle \geq 0$.
\end{proof}

Using \cref{lem:bounded} and the same assumptions as in~\cite[Assumptions 1, 2 and 3]{cao2020fltrust}, the proof of convergence of \ourscheme follows the same steps as the proofs in Appendix A of~\cite{cao2020fltrust}.

\bibliographystyle{IEEEtran}
\bibliography{IEEEabrv,refs}

\clearpage

\end{document}